\documentclass[11pt,twoside]{article}

\usepackage[T1]{fontenc}
\usepackage{amsmath,amssymb,amsthm}
\usepackage{graphicx}
\usepackage{geometry}
\usepackage{xcolor}
\usepackage{enumitem}
\usepackage{booktabs}
\IfFileExists{newtxtext.sty}{%
  \usepackage{newtxtext,newtxmath}%
}{%
  \usepackage[sc]{mathpazo}%
  \usepackage{mathrsfs}%
}
\usepackage{microtype}
\usepackage{setspace}
\usepackage{makecell}
\usepackage{subcaption}
\usepackage{multirow}
\usepackage{titlesec}
\usepackage{caption}
\usepackage{fancyhdr}
\usepackage[authoryear]{natbib}
\usepackage{hyperref}

\definecolor{linkblue}{RGB}{31,66,102}
\hypersetup{
  colorlinks=true,
  linkcolor=linkblue,
  citecolor=linkblue,
  urlcolor=linkblue,
  pdftitle={Diffusion Models for High-Dimensional Clustered Data: Intrinsic-Dimension Adaptivity via Bayesian Classification},
  pdfauthor={Yuga Iguchi and Paul Fearnhead}
}

\titleformat{\section}
  {\normalfont\large\scshape\centering}
  {\thesection.}{0.6em}{}
\titleformat{\subsection}
  {\normalfont\normalsize\bfseries}
  {\thesubsection.}{0.55em}{}
\titleformat{\subsubsection}
  {\normalfont\normalsize\itshape}
  {\thesubsubsection.}{0.55em}{}
\titlespacing*{\section}{0pt}{2.0ex plus 0.5ex minus 0.2ex}{1.0ex plus 0.2ex}
\titlespacing*{\subsection}{0pt}{1.7ex plus 0.4ex minus 0.2ex}{0.7ex plus 0.2ex}
\titlespacing*{\subsubsection}{0pt}{1.5ex plus 0.3ex minus 0.2ex}{0.5ex plus 0.2ex}

\fancypagestyle{plain}{%
  \fancyhf{}%
  \fancyfoot[C]{\small\thepage}%
}

\newtheoremstyle{imsplain}
  {8pt plus 2pt minus 1pt}{8pt plus 2pt minus 1pt}
  {\itshape}{0pt}{\scshape}{.}{0.6em}{}
\newtheoremstyle{imsdefinition}
  {8pt plus 2pt minus 1pt}{8pt plus 2pt minus 1pt}
  {\normalfont}{0pt}{\scshape}{.}{0.6em}{}
\newtheoremstyle{imsremark}
  {7pt plus 2pt minus 1pt}{7pt plus 2pt minus 1pt}
  {\normalfont}{0pt}{\itshape}{.}{0.6em}{}

\theoremstyle{imsplain}
\newtheorem{theorem}{Theorem}
\newtheorem{lemma}{Lemma}[section]

\newtheorem{proposition}{Proposition}
\newtheorem{corollary}{Corollary}
\theoremstyle{imsdefinition}
\newtheorem{assump}{Assumption}
\theoremstyle{imsremark}
\newtheorem{remark}{Remark}

\newcommand*{\snr}[1]{\mathrm{SNR}_{#1}}  
\newcommand*{\for}[2]{\overrightarrow{#1}_{#2}} 
\newcommand*{\back}[2]{\overleftarrow{#1}_{#2}} 
  
\newcommand*{\kl}{\mathrm{KL}}   
\newcommand*{\KL}[2]{\mathrm{KL} \bigl( {#1} \, || \, {#2} \bigr) }   
\newcommand*{\KLM}[2]{\mathrm{KL}_m \bigl( {#1} \, || \, {#2} \bigr) }   
\newcommand*{\KLC}[2]{\mathrm{KL}_c \bigl( {#1} \, || \, {#2} \bigr) }

\newcommand*{\tr}{\mathrm{tr}}
\newcommand*{\Cone}{C_1}
\newcommand*{\Ctwo}{C_2}
\newcommand*{\Cfour}{C_3}
\newcommand*{\Cfive}{C_4}
\newcommand*{\Csix}{C_5}

\newcommand{\data}{{P_{\mathrm{data}}}} 
\title{Diffusion Models for High-Dimensional Clustered Data: Intrinsic-Dimension Adaptivity via Bayesian Classification}
\author{Yuga Iguchi \and Paul Fearnhead}
\newcommand{\paperaffiliation}{School of Mathematical Sciences, Lancaster University, Lancaster, UK}
\date{\today}

\makeatletter
\renewcommand{\maketitle}{%
  \thispagestyle{plain}%
  \noindent{\footnotesize\textsc{Research Preprint}\hfill\@date}\par
  \vspace{2.35em}%
  \begin{center}
    {\fontsize{16}{20}\selectfont\bfseries\MakeUppercase{\@title}\par}%
    \vspace{1.45em}%
    {\normalsize\textsc{By Yuga Iguchi and Paul Fearnhead}\par}%
    \vspace{0.55em}%
    {\small\itshape\paperaffiliation\par}%
  \end{center}
  \vspace{1.05em}%
}
\renewenvironment{abstract}{%
  \begin{list}{}{%
    \setlength{\leftmargin}{0.065\textwidth}%
    \setlength{\rightmargin}{0.065\textwidth}%
    \setlength{\topsep}{0pt}%
  }%
  \item\relax\small
  \begin{center}\textsc{Abstract}\end{center}%
  \noindent\ignorespaces
}{%
  \end{list}\vspace{0.55em}%
}
\makeatother

\begin{document}
\maketitle 
\begin{abstract}
The empirical success of diffusion models in generative modelling has motivated theoretical work, including quantitative error bounds and qualitative analyses that characterise the different phases of denoising. We bring these two areas together by studying the adaptivity of diffusion models to the structured geometry of multimodal high-dimensional data that consists of multiple clusters in $\mathbb{R}^D$, each with its own low-dimensional structure, and inter-cluster separation depending on $D$. We employ $K$-mixture Gaussian distributions as a canonical framework to capture this geometry and establish two theoretical results. First, we interpret denoising as a dynamical Bayesian classifier: the mixture score is a posterior-weighted average of cluster-wise scores, and we show that, with high probability, the posterior class probabilities concentrate on a single cluster once the signal-to-noise ratio reaches the scale $\Theta (\log (KD)/D)$. Second, by separately analysing the denoising process in its mixing and cluster-commitment phases, we prove that the KL error bound depends linearly on the maximum intrinsic dimension of a cluster, up to a logarithmic factor, even when $K$ grows polynomially with $D$. This improves on ambient-dimensional bounds and extends existing low-dimensional adaptivity analyses to multimodal distributions with heterogeneous, approximately low-rank covariances. 
\end{abstract}


\section{Introduction}

Denoising Diffusion Probabilistic Models (DDPMs) \citep{Song:2019,ho:20,song:20} are general methods for generative modelling that have produced state-of-the-art results across many applications \citep{dhariwal2021diffusion,popov21a,lugmayr2022repaint,trippe2023diffusion}. They define a process that takes data and progressively adds noise to it to produce a realisation from (a close approximation) to some simple distribution. They then learn the reverse, denoising, process, so that we can generate new data examples by denoising realisations from the simple distribution. 

Relatedly, there has been much interest in understanding why DDPMs are so effective. Much of this work has been theoretical, producing bounds on the accuracy of DDPM. Initial work gave bounds in terms of the Kullback-Leibler (KL) divergence between the true data distribution and the distribution sampled by a DDPM \citep{ben:24,conf:25}, though these have been extended to other metrics such as total variation distance \citep{liang25a,brevsar2025nonasymptotic}, 1-Wasserstein distance \citep{de:22} or 2-Wasserstein distance \citep{GNHZ:25, Sil:25, AED:25, koike2026wasserstein}. 

Within this body of work there has been particular interest in understanding why DDPMs scale well to high-dimensional distributions. Under KL divergence, standard bounds on the accuracy of DDPMs \citep{ben:24, conf:25} have terms that, up to a logarithmic factor, scale linearly with the ambient dimension, $D$,  and suggest that the computational cost, such as the number of time-steps used when denoising, needs to scale with $D$. While earlier works have established such a bound under minimal conditions on the target, e.g., bounded second moment for data distribution, more recent work has made assumptions that the data lie on a low-dimensional manifold with a bounded support and then shows that in such cases the error can scale with the dimension of the manifold rather than the ambient dimension.
For example, \cite{pota:25} assume that data lies on a bounded manifold of intrinsic dimension $k$ and show that the KL-bound is then linear  in $k$.  \cite{Li:24} define the dimension of the manifold, $k$, in terms of metric entropy (which implicitly also implies the data distribution has bounded support) and show that to bound the total variation error the number of time-steps when denoising needs to scale quadratically in $k$. This was subsequently improved to linearly in $k$ by \cite{Li:25} and \cite{Huang:26}.

Separately, work has focussed on whether DDPMs generalise or just memorise the data \citep{Li2023,Dupuis2025,Bonnaire2025, maillard2026}. That is, does simulating from a DDPM produce new examples from the underlying data distribution, or just reproduce members of the data set it was trained on. Many applications of DDPMs involve data distributions with separate clusters, for example image data with images of different types. \cite{BBDM:24} studied the noising/denoising process for such data and was able to give rigorous qualitative insight into its behaviour albeit assuming exact simulation of the denoising process using the true empirical score function. They show that the denoising process can be split into three phases, one where there is general exploration, a second where the process fixes the cluster it will sample from, and then a final stage where the process will be attracted to a specific data point from the training sample. The last stage corresponds to memorisation, and can be avoided by early stopping of the denoiser. \cite{Ach:26} further extended this study of DDPMs for mixture distributions. 
\subsection{Our Contributions}
Our focus is on analysing the geometric adaptivity of DDPMs for target data from a $K$-mixture of distributions on $\mathbb{R}^D$, exhibiting both low dimensionality and distinct clustering. This is motivated by the notion of the \emph{union of manifolds hypothesis} \citep{brown:22}: that high-dimensional data such as images lie on a disjoint union of manifolds of varying intrinsic dimensions. Our contributions are summarised as follows:

\begin{itemize}[leftmargin=0.2cm] 
\item We rigorously extend the general framework of \cite{BBDM:24} and \cite{Ach:26} to study a situation where each cluster distribution samples from, or is close to, a low-dimensional manifold. Under the structured geometry specified later, we show that the denoising process fixes on a single cluster to sample from with high probability after a suitably defined signal-to-noise ratio (SNR) reaches a threshold $\Theta (\log (KD)/D)$. 
\item
Exploiting the above characterisation of cluster commitment in terms of SNR, we give rigorous bounds on the error of DDPM accounting for the various approximations such as time-discretisation of the denoising process, which, under suitable regularity conditions, is shown to scale linearly with the maximum intrinsic dimension of any of the clusters, rather than the ambient dimension $D$. 
\end{itemize}
As such our work brings together the mixture framework of \citep{BBDM:24} and \cite{Ach:26} with the literature on KL-bounds for DDPMs, and the ability of DDPMs to adapt to low-dimensional structure within each distinct cluster. To demonstrate the intrinsic adaptivity in the KL error despite the distinct multimodality, 
we first characterise the 
phase transition where the DDPM switches from exploring between clusters to fixing on a single cluster.  
By leveraging this, and separately analysing the behaviour of the DDPM before and after the phase transition, we are able to show that the KL error bound depends linearly on the intrinsic dimensions up to a logarithmic factor of $D$, even when the number of clusters grows polynomially in the ambient dimension $D$. A detailed description of the link between the phase transition and the intrinsic dimension adaptivity in KL error is provided in Section \ref{sec:overview}. 

In order to get our results, we make a number of assumptions on the data distribution. The strongest of which is that each component of our mixture is Gaussian. i.e., Gaussian mixture models (GMMs). This is in line with recent papers that study simplified models in order to gain insight into DDPMs \citep{li:26, oymak2021theoretical,shah2023learning, wu2024theoretical, Wang:25}. 
In particular, \cite{li:26} 
establish dimension-free convergence of DDPM in total variation under the common isotropic covariance structure. In contrast, we consider GMMs with heterogeneous and anisotropic covariances and assume that each cluster is well-separated, in that the squared distances between cluster means scale linearly with $D$, while each cluster has its own low-dimensional structure. This inter-cluster separation induces the phase transition described above, leading to our intrinsic-dimension adaptivity result under the heterogeneous and anisotropic covariance geometry. In our work, the intrinsic dimensionality is imposed by an assumption that the spectrum of the cluster covariance matrix decays quickly: for cluster $i$, there is some $M_i$ such that the sum of all but the $M_i$-largest eigenvalues is bounded. This is a weaker condition than assuming the data in cluster $i$ lies on an $M_i$-dimensional manifold. Importantly, we do not require that the support of the data distribution is bounded. 
We stress that our main result focuses on the dimensional scaling of the noise or time-discretisation schedule for DDPMs or associated reverse SDEs, given a sufficiently accurate estimate of the score, which analytically explains why a reasonable number of steps is enough in practice to sample from high-dimensional clustered data. In contrast, \cite{Wang:25} focuses on score estimation, and clarifies the relation between the training samples and the intrinsic dimensionality to minimise a training loss when the target is a mixture of exactly low-rank Gaussians.  

This paper is organised as follows. Section \ref{sec:setup} summarises the basic set-up of diffusion models and introduces the structured geometry under GMMs. Section \ref{sec:phase_transition} presents our two main analytical results on the phase transition and intrinsic-dimension adaptivity in KL error. We provide empirical experiments using real data in Section \ref{sec:numerical_experiments} to demonstrate that our analytical findings on the phase transition can be observed beyond GMMs. We prove the second main result in Section \ref{sec:pf_main} and then conclude in Section \ref{sec:conclusion}. 
\\ 

\noindent \textit{Notation.} For a vector $x \in \mathbb{R}^D$ and matrix $A \in \mathbb{R}^{D \times D}$, we write $\| x \|_A = \sqrt{x^\top A x}$ and $\| x \| = \sqrt{\sum_{i = 1}^D x_i^2}$. 
Denote by $\lambda_i (A)$ the $i$-th largest eigenvalue of the matrix A and write $\| A \|_2 = \sqrt{\lambda_{1} (A^\top A)}$ and $\| A \|_F = \sqrt{\tr (A^\top A)}$. 
For two vectors $x, y \in \mathbb{R}^D$, $\langle x, y \rangle$ denotes the inner product. For two matrices $A, B \in \mathbb{R}^{D \times D}$, $A \preceq B$ mean that $B - A$ is positive semidefinite. 
For positive integer $M$, we write $[M] \equiv \{1, \ldots, M\}$. 
For $ x > 0$, we write $f(x) = \Theta (x)$ if there exist constants $0 < c_1 \le c_2 < \infty$ such that $c_1 x \le f(x) \le c_2 x$. 
For two sequences $\{a_n\}_n$ and $\{b_n\}_n$, we write $a_n = \mathcal{O} (b_n)$ if there exists a constant $C > 0$ such that $a_n \le C b_n$ for all sufficiently large $n$. We also write $a_n = o (b_n)$ if $\lim_{n \to \infty} a_n / b_n = 0$. 
For any two probability measures $\mathbb{P}$ and $\mathbb{Q}$ with $\mathbb{P}$ being absolutely continuous with respect to $\mathbb{Q}$, we write $\KL{\mathbb{P}}{\mathbb{Q}} = \int \log \tfrac{d \mathbb{P}}{d \mathbb{Q}} (\omega)  d \mathbb{P} (\omega) $ for the KL divergence between them. We write $\gamma^D : \mathcal{B} (\mathbb{R}^D) \to [0, 1]$ as the standard Gaussian measure defined as $\gamma^D (A) = (2\pi)^{-D/2}  \int_A \exp (- \tfrac{\|z \|^2}{2}) dz, \, A \in \mathcal{B} (\mathbb{R}^D)$. Also, for $m \in \mathbb{R}^D$ and $V \in \mathbb{R}^{D \times D}$, $\mathcal{N} (m, V)$ denotes the normal distribution with mean $m$ and covariance $V$.  For a random variable $X$, $\mathrm{Law} (X)$ denotes the probability law of $X$. 

\section{Setup} \label{sec:setup}
\subsection{Score-Based Generative Models}
We first review score-based generative models \citep{song:20}, which provide a continuous-time framework for sampling from a target distribution. Let $(\Omega, \mathcal{F}, \mathbb{P})$ be the probability space where the $D$-dimensional standard Brownian motion $(B_t)_{t \ge 0}$ is equipped, and $\data$ denote the target distribution on $\mathbb{R}^D$. Denoising diffusion models are designed to generate samples from $P_{\mathrm{data}}$ by reversing a forward noising process, where the latter process can be formulated as the following Ornstein-Uhlenbeck (OU)-type SDE: 
\begin{align*}
d \for{X}{t} & = - \for{X}{t} dt + \sqrt{2} d B_t, \quad \for{X}{0} \sim \data, \quad t \in [0, T],   
\end{align*} 
where $T > 0$ is a fixed time horizon. We denote by $\for{p}{t} : \mathcal{B} (\mathbb{R}^D) \to [0, 1]$ the marginal distribution of $\for{X}{t}$, given as
\begin{align*}
\for{p}{t} (A) =   
\int \for{p}{t|0} (A | x) P_{\mathrm{data}} (dx), \quad A \in \mathcal{B} (\mathbb{R}^D),   
\end{align*}
where $\for{p}{t|0} (A | x) \equiv \mathbb{P} \bigl(  \for{X}{t} \in A \, | \, \for{X}{0} = x \bigr)$. The forward noise $\for{X}{t}$ conditional on the initial data point $\for{X}{0}$ follows a Gaussian distribution, and specifically, the solution is expressed as 
$
\for{X}{t} = m_t \for{X}{0} + \sigma_t Z, ~Z \sim \mathcal{N}(0, I_D), 
$
where $m_t$ and $\sigma_t$ represent the strength of signal and noise, defined as 
\begin{align}  \label{eq:signal_noise}
m_t = \exp (-t), \quad \sigma_t = \sqrt{1 - m_t^2}.  
\end{align}  
With a slight abuse of notation, we write $ \for{p}{t|0}$ as the associated Lebesgue density. Similarly, we write $\for{p}{t}$ for the Lebesgue density of the marginal law of $\for{X}{t}$ if it exists. Then, the denoising SDE is defined as: 
\begin{align} 
d \back{X}{t} & = \bigl(
\back{X}{t} + 2 \nabla \log \for{p}{T-t} (\back{X}{t}) \bigr) dt + \sqrt{2}  d {B}_t, \quad \back{X}{0} \sim \mathrm{Law} (\for{X}{T}). \label{eq:reverse_sde} 
\end{align}
The classical results of \citep{Anderson:82, Haussmann:86} show that $(\back{X}{t})_{t \in [0, T]}$ admits a weak solution so that $\mathrm{Law} (\back{X}{t}) = \mathrm{Law} (\for{X}{T-t})$ for any $t \in [0, T]$. In particular $\mathrm{Law}(\back{X}{T}) = \mathrm{Law}(\for{X}{0})$, so exact simulation of the reverse SDE (\ref{eq:reverse_sde}) allows for sampling from the target distribution $\data$. 

DDPMs \citep{ho:20} can be viewed as a tractable alternative to (\ref{eq:reverse_sde}), which involves three approximations: (i) The use of an estimated score --  the intractable score function is replaced with a function $\hat{s} : [0, T] \times \mathbb{R}^D \to \mathbb{R}^D$ parametrised with a neural network, obtained by optimising an appropriate loss function \cite[]{yang2023diffusion}; (ii) The use of the Gaussian prior $\mathcal{N} (0, I_D)$ for initialisation rather than $\mathrm{Law} (\for{X}{T})$; and (iii) Time-discretisation of the continuous-time process (\ref{eq:reverse_sde}). 


Even with these approximations, it can be difficult to approximately simulate from (\ref{eq:reverse_sde}) for $t\approx T$ as the drift, specifically the term $2 \nabla \log \for{p}{T-t} (\back{X}{t})$, blows-up as $t\rightarrow T$. To overcome this, simulation is often performed up to some time $T-\varepsilon$, for a suitable choice of early-stopping parameter $\varepsilon>0$. It is simple to bound the error between $\mathrm{Law}(\back{X}{T-\varepsilon})$ and $\mathrm{Law}(\back{X}{T})$ in terms of $\epsilon$ \cite[]{ben:24}. 

%
\subsection{Bayesian Classification Perspective and Our Analytic Framework} \label{sec:geometry} 
Given the above setup, our primary goal is to rigorously investigate the generative mechanism and geometric adaptivity of the reverse process, characterised by (\ref{eq:reverse_sde}), when $P_{\mathrm{data}}$ is a union of $K$ clusters, written in the form 
\begin{align*}
P_{\mathrm{data}} (A) = \sum_{k = 1}^K \pi_k \cdot P_{\mathrm{data}}^{[k]} (A), \quad K \ge 2, \, A \in \mathcal{B} (\mathbb{R}^D), 
\end{align*} 
where $\{\pi_k\}$ satisfying $\sum_{k = 1}^K \pi_k = 1$ are the prior weights and $P_{\mathrm{data}}^{[k]}$ denotes the distribution associated with the $k$th class. In the mixture setting, the key observation is that the denoising procedure can be viewed as a dynamical Bayesian classification, as advocated also in recent work \citep{Ach:26}. More specifically, the score function $\nabla \log \for{p}{T-t} (\cdot)$ contained in the drift of (\ref{eq:reverse_sde}) is expressed as a weighted average of the cluster-wise scores, that is,  
\begin{align*}
\nabla \log \for{p}{T-t} (x) = \sum_{k = 1}^K W_D^{[k]} (t, x) \nabla \log \for{p}{T-t}^{[k]} (x), \qquad x \in \mathbb{R}^D,  
\end{align*}
where $\for{p}{T-t}^{[k]} (x)$ is the Lebesgue density of $\for{X}{T-t}$ given initial data point $\for{X}{0}$ belonging to the $k$-th cluster, i.e., 
\begin{align*}  
\for{p}{T-t}^{[k]} (x) = \int \for{p}{T-t|0} (x | z) \cdot P^{[k]}_{\mathrm{data}} (dz)
\end{align*}
and
$W_D^{[k]} (t, x)$ is the \emph{posterior mixture weight} or \emph{Bayes classifier}, i.e., the conditional probability of the label taking $k$ given the state $\for{X}{t} = x$, specifically, 
\begin{align} \label{eq:posterior_weight}
W_D^{[k]} (t, x) = \frac{\pi_k \cdot \for{p}{T-t}^{[k]} (x)}{\sum_{\ell=1}^K \pi_\ell \cdot \for{p}{T-t}^{[\ell]} (x)}.   
\end{align} 
Thus, the reverse process (\ref{eq:reverse_sde}) evolves according to the posterior weights given the current denoising state $\back{X}{t}$, indicating that it implicitly performs a dynamic Bayesian classification within the clustered data setting. 

The insight presented above suggests that understanding the generative mechanism of diffusion models for clustered data relies on studying the stochastic behaviour of the posterior weights $W_D^{[k]}(t, \back{X}{t}) $ for all $t \in [0, T)$. In this paper, we focus on \emph{Gaussian mixture models} (GMMs) as the target data, as they provide a solid analytical framework for rigorously examining the dynamic Bayesian classification. 
Under this setting, the marginal distribution of the forward noise also follows a Gaussian mixture. Also, cluster-wise low-dimensional structure can be employed within each covariance. Thus, we define $P_{\mathrm{data}}^{[k]} (A) = \mathcal{N} (A; \mu_k, \Sigma_k), \, A \in \mathcal{B}(\mathbb{R}^D)$ and specify the target distribution as:
\begin{align} \label{eq:gmm}
P_{\mathrm{data}} (A) 
= \sum_{k = 1}^K  \pi_k \cdot \mathcal{N} (A ; \mu_k, \Sigma_k), 
\qquad K \ge 2, \ A \in \mathcal{B} (\mathbb{R}^D),    
\end{align}
where $\mu_k \in \mathbb{R}^D$, $\Sigma_k \in \mathbb{R}^{D \times D}$ with $\Sigma_k \succeq 0$ and $\mathcal{N} (A; \mu_k, \Sigma_k)$ is a Gaussian measure on $\mathbb{R}^D$ defined as 
\begin{align*}
\mathcal{N} (A ; \mu_k, \Sigma_k) \equiv \gamma^D 
\bigl( \{ z \in \mathbb{R}^D \, | \, \mu_k + \Sigma_k^{1/2} z \in A \} \bigr). 
\end{align*}
The measure admits a well-defined Lebesgue density if  $\Sigma_k$ is positive definite. Also, note that the distribution of the cluster-wise forward noising $\for{p}{T-t}^{[k]} (\cdot), \, t \in [0, T), \, k \in [K]$ follows a Gaussian with mean $m_{T-t} \mu_k$ and covariance $S_{T-t}^{k}$ given as 
\begin{align} \label{eq:forward_noise_cov}
S_{T-t}^{k} = m_{T-t}^2 \Sigma_k + \sigma_{T-t}^2 I_D    
\end{align}
which is positive definite due to the added noise for $t \in [0, T)$.

We now introduce the following structure for the target measure (\ref{eq:gmm}): 
\begin{assump}[Prior Weights]  \label{assump:prior}
There exists a constant $\Cone \ge 0$ such that 
$$ 
\max_{i, j \in [K]} \log \tfrac{\pi_i}{\pi_j} \le \Cone. 
$$ 
\end{assump}
\begin{assump}[Covariance]  \label{assump:cov}

\begin{itemize}
\item[(i)] $\Sigma_i$ is symmetric positive {semi-definite} for any cluster $i \in [K]$. Also, there exist a constant $\Ctwo > 0$ and $\alpha \in [0,1)$ such that for all $D \in \mathbb{N}$, 
\begin{align*}
\max_{i \in [K] } \| \Sigma_i \|_2 \le \Ctwo \cdot D^\alpha. 
\end{align*}
\item[(ii)] (Low-dimensional structure) For each cluster $i \in [K]$, there exist a positive fixed integer $M_i < D$ and ${\nu_{i} \ge 0}$ such that for all $D \in \mathbb{N}$, we have 
\begin{align} \label{eq:assum2ii}
\sum_{k = M_i + 1}^D \lambda_k (\Sigma_i) \le \nu_i,  
\qquad \frac{\nu_i}{M_i} \le R
\end{align} 
for some universal constant $R \in [0,1)$.  
\end{itemize}
\end{assump}

\begin{assump}[Mean] \label{assump:mean}
For any two clusters $i, j \in [K]$, there exist constants $\Cfour, \Cfive > 0$ such that for all $D \in \mathbb{N}$,  
\begin{align} \label{eq:mean_gap}
\Cfour \cdot D \le \| \mu_i - \mu_j \|^2  \le \Cfive \cdot D.
\end{align}  
%
%
\end{assump}

The first assumption is weak as it just bounds the ratio of the mixture weights.  The second assumption imposes the structure of each component. Assumption \ref{assump:cov}(i) allows at most sub-linear dependency on $D$ for the maximum eigenvalue of each cluster's covariance, which is weaker than assuming a uniform bound w.r.t. $D$. Sublinear growth is a technical condition that facilitates the phase transition analysis later in Section \ref{sec:phase_transition}. Assumption \ref{assump:cov}(ii) together with (i) covers both exact or approximately low-rank structure; if data from component $i$ is confined to a manifold of dimension $M_i$, then (\ref{eq:assum2ii}) would hold with $\nu_i=0$. Notably, Assumption \ref{assump:cov}(ii) is a weaker condition that only requires data from each component to lie close to a low-dimensional manifold, with $R$ being a global measure of how close data from each mixture component is to its manifold. 

Finally, Assumption \ref{assump:mean} defines the separation of clusters as a function of the ambient dimension $D$. The square distance between the cluster means is assumed to increase linearly with $D$. 
It is also natural for image data where increasing $D$ corresponds to observing the same image at increasing resolution. We will also empirically assess these data geometries using high-dimensional datasets (images and RNA sequencing data) in Section \ref{sec:numerical_experiments}.

\section{Posterior Classification Analysis and Intrinsic Dimension Adaptivity} \label{sec:phase_transition} 

\subsection{Overview} \label{sec:overview}

This section provides an overview of the generative mechanism of the reverse process within the GMM framework (\ref{eq:gmm}) detailed in Section \ref{sec:geometry}. We will also explain how the dynamic Bayesian classification perspective can be linked to the KL convergence analysis, thereby demonstrating its ability to adapt to the geometry of the low-dimensional structure within each cluster. 
 
Sections \ref{sec:posterior_1}--\ref{sec:posterior_3} analyse the stochastic behaviour of the posterior weights $W_D^{[k]} (t, \back{X}{t})$ given in \eqref{eq:posterior_weight}, revealing that the denoising process decomposes into the following two phases: In the Phase I -- the mixing phase --  several posterior weights remain non-negligible, and the trajectory is affected by global inter-cluster geometry. Then, in Phase II -- the commitment phase --, one posterior weight becomes dominant, that is the posterior weight concentrates on a single class, i.e., $W_D^{[k]} \approx 1$ for some $j \in [K]$ and $W_D^{[i]} \approx 0,\, i \in [K] \setminus \{j\}$, and the score is effectively reduced to the corresponding cluster-wise score
\begin{align*}
\nabla \log \for{p}{T-t} (\back{X}{t}) 
\approx \nabla \log \for{p}{T-t}^{[j]} (\back{X}{t}). 
\end{align*} 
The analytic results presented in Sections \ref{sec:posterior_2}--\ref{sec:posterior_3} demonstrate a key effect of the ambient dimension $D$ on the posterior weight concentration. Roughly, the transition from Phase I to II or mode commitment can occur with probability at least $1 - D^{-q}, \ q \ge 1$ after the signal-to-noise ratio (SNR) 
\begin{align} \label{eq:snr}
\snr{T-t} \equiv \tfrac{m_{T-t}^2}{\sigma_{T-t}^2}, \qquad t \in [0, T), 
\end{align}
reaches scale $\log (KD) / D$, see Theorem \ref{thm:NA_transition_est} in Section \ref{sec:posterior_3}. These observations imply that inter-class separation that scales with the ambient dimension $D$ facilitates earlier resolution of the cluster component identification, and the rest of the denoising process then focuses on building fine details from the single cluster. Table \ref{table:regimes_summary} summarises this two-phase picture.
\begin{table}[h]
\footnotesize  
\centering
\caption{Summary of the dynamical regimes of the denoising diffusion for high-dimensional $K$-clustered data.} 
\label{table:regimes_summary}
\renewcommand{\arraystretch}{1.2} 
\setlength{\tabcolsep}{4pt}
\begin{tabular}{@{} l @{\hspace{1.2em}} p{0.32\linewidth} @{\hspace{0.2em}} p{0.32\linewidth} @{}}
\toprule
& \textbf{Phase I} & \textbf{Phase II} \\
\midrule

\textbf{Phase Characterisation} & 
Mixing phase & 
Commitment to a single cluster $j$  \\
\addlinespace

\textbf{Signal-to-noise ratio} 
($\snr{T-t}$) & 
$\mathcal{O} \bigl( {\log (KD)}/{D} \bigr)$ 
& 
$\mathcal{O}(1)$ \\
\addlinespace 

\textbf{Posterior Weights} & 
$W_D^{[k]} \ge 0, \, k \in [K] $ \newline {{(Onset of posterior collapse)}} & 
$W_D^{[j]} \approx 1, \, W_D^{[k]} \approx 0,\, k \neq j$ \newline {{(Posterior concentration)}} \\
\addlinespace

\textbf{Effective Geometry} & 
\textit{Inter-cluster inference} 
\newline + \textit{local covariances} & 
\textit{Single local covariance} {{(cluster $j$)}} \\
\bottomrule
\end{tabular} 
\end{table}     

Given the above characterisation of the denoising process, we bound the KL error between the target distribution and its approximation realised by DDPM samplers in Section \ref{sec:KL}. The main analytic result therein is that, up to a logarithmic factor, the time-discretisation error component of the KL divergence scales linearly with the intrinsic dimension, i.e., the low-rank property within each cluster, rather than with the ambient dimension $D$. To show this we use different approaches to bound the KL error in Phase I and Phase II: For Phase I we use the fact that the KL error involves terms that depend on the SNR, and in this phase the SNR is $\mathcal{O} (\log (KD) / D)$, which offsets the inter-cluster separation depending on $D$ in the first and second moments. For Phase II, due to cluster-commitment, the denoiser behaves as one targeting a single cluster, so we can bound the error in terms of the cluster's intrinsic dimension. A more detailed overview of our argument is given in Section \ref{sec:poc}.


\begin{remark}
Similar attempts to split the denoising process into distinct phases are considered in \citep{BBDM:24, Ach:26}. One of the focuses in these works is on characterising when the marginal density of the process transitions from a single-well structure (when the process is contractive) to a multi-well one. They study this in an asymptotic manner, i.e., consider a Taylor series expansion of the score function in very small signal $m_{T-t} \ll 1$ and analyse the curvature of the approximate Hessian. In contrast, our analysis proceeds with a probabilistic characterisation of posterior weights concentration without relying on the Taylor expansion in $m_{T-t}$.  
\end{remark}   

\subsection{High Probability Bounds for Posterior Weights}
 \label{sec:posterior_1}
 We begin by stating the following dynamic characterisation of the posterior weights, which serves as a basis to develop the phase transition arguments hereafter:  
\begin{lemma} \label{lemma:weight_bd_j}
Let $T \ge 1$ and $\delta \in (0,1)$. Then, for each $i, j \in [K]$, it holds that: 
\begin{align*}
\mathbb{P} \Bigl( \log W_D^{[i]} (t, \back{X}{t}) 
\le \min \bigl\{ - C_D^{[i, j]} (t, \delta) , 0 \bigr\}
\, | \, \back{X}{T} \sim P_{\mathrm{data}}^{[j]} \Bigr) \ge 1 - \delta 
\end{align*}
for each $t \in [0, T)$, where we have set:
\begin{align}
\begin{aligned}
C_D^{[i, j]} (t, \delta) 
& = \log \tfrac{\pi_j}{\pi_i} 
+ \KL{\, \for{p}{T-t}^{[j]}}{\, \for{p}{T-t}^{[i]}} \\[0.2cm]
& \qquad 
-  \sqrt{2 \mathscr{T}_D^{[i, j]} (t) \log \tfrac{2}{\delta}}
- \sqrt{2 \mathscr{U}_D^{[i, j]} (t) \log \tfrac{2}{\delta}}  
-  2 \mathscr{V}_D^{[i, j]} (t) \log \tfrac{2}{\delta},  
\end{aligned} \label{eq:C}
\end{align}
with 
\begin{align*}
\mathscr{T}_D^{[i, j]} (t) 
& = m_{T-t}^2 \bigl\| \mu_{i} - \mu_{j} \bigr\|_{(S_{T-t}^{i})^{-1} (S_{T-t}^{j}) (S_{T-t}^{i})^{-1}}^2; \\[0.1cm] 
\mathscr{U}_D^{[i, j]}(t) 
& = \tfrac{1}{2} \mathrm{tr} \Bigl( \bigl( I_D - S_{T-t}^{j} (S_{T-t}^{i})^{-1} \bigr)^2 \Bigr); \\[0.1cm]
\mathscr{V}_D^{[i, j]}(t) & = \| I_D - S_{T-t}^j (S_{T-t}^i)^{-1}  \|_{2}.  
\end{align*} 
\end{lemma} 
%
%
The proof of Lemma \ref{lemma:weight_bd_j} is postponed to Appendix \ref{sec:pf_weight_bd}. We further introduce event sets   
\begin{align}
A^{[i, j]} (t, \delta) & = \Bigl\{ \omega \in \Omega \ | \ 
\log W_D^{[i]} (t, \back{X}{t}) \le \min \bigl\{ - C_D^{[i, j]} (t, \delta), 0 \bigr\} \Bigr\}, \qquad
i, j \in [K].  
\label{eq:set_A}
\end{align} 
Then, the following result is immediate: 
\begin{proposition} \label{prop:NA_weight_bd}
Let $T \ge 1$, $\delta \in (0,1)$ and $j \in [K]$. It holds that:
\begin{align*}
\mathbb{P} \Bigl( \bigcap_{i \in [K]} A^{[i,j]} \bigl(t, \tfrac{\delta}{K} \bigr) \, | \, \back{X}{T} \sim  P_{\mathrm{data}}^{[j]} \Bigr) \ge 1 - \delta,  
\end{align*}
for each $t \in [0, T)$. 
\end{proposition}
Proposition \ref{prop:NA_weight_bd} tells us that 
for denoising trajectories $\{ \back{X}{t} \}$ eventually sampling from the cluster $j$ at the terminal time $T$, the posterior weights $W_D^{[i]} (t, \back{X}{t})$ with any cluster index $i$ rather than $j$ are bounded by $\min \bigl\{\exp \bigl(- C_D^{[i, j]} (t, \delta/K) \bigr), 1 \bigr\}$ with probability at least $1 - \delta$,  $\delta \in (0,1)$. 
%
\\ 

\noindent 
\textit{Proof of Proposition \ref{prop:NA_weight_bd}.}
We have that: 
\begin{align*}
& \mathbb{P} \Bigl(  \Bigl( \bigcap_{i \in [K]} A^{[i,j]} \bigl(t, \tfrac{\delta}{K} \bigr) \Bigr)^c \ | \ \back{X}{T} \sim  P_{\mathrm{data}}^{[j]} \Bigr)
= \, \mathbb{P} \Bigl(  \bigcup_{i \in [K] } \Bigl( A^{[i,j]} \bigl(t, \tfrac{\delta}{K} \bigr) \Bigr)^c \ | \ \back{X}{T} \sim  P_{\mathrm{data}}^{[j]} \Bigr)
\\ 
& \qquad \le 
\sum_{i \in [K]} \mathbb{P} \Bigl( \Bigl( A^{[i,j]} \bigl(t, \tfrac{\delta}{K} \bigr) \Bigr)^c \ | \ \back{X}{T} \sim  P_{\mathrm{data}}^{[j]} \Bigr) 
\le \sum_{i \in [K]} \tfrac{\delta}{K} = \delta, 
\end{align*} 
where we used Lemma \ref{lemma:weight_bd_j} in the last line. The proof is complete.

\subsection{Posterior Weights in High-Dimensional Limit}
\label{sec:posterior_2}
Proposition \ref{prop:NA_weight_bd} indicates that the decay of the posterior weights is dynamically controlled by how large the signal $\KL{\, \for{p}{T-t}^{[j]}}{\, \for{p}{T-t}^{[i]}}$ is compared to the prior ratio and the fluctuation terms. The following result demonstrates how the signal component dominates in the term $C_D^{[i,j]}$ in the high-ambient-dimensional regime, $D \to \infty$. 
%
\begin{proposition} \label{prop:LDP_posterior}
Let $T \ge 1$ and $j \in [K] $.  
Also, let Assumptions \ref{assump:prior}, \ref{assump:cov}, \ref{assump:mean} hold and assume that for each $t \in [0, T)$ and $i, j \in [K]$, the following limit exists: 
\begin{gather*}
\lim_{D \to \infty} \tfrac{1}{D} \times 
\KL{\, \for{p}{T-t}^{[j]}}{\, \for{p}{T-t}^{[i]}} 
\equiv \widetilde{C}_\infty^{\, [i, j]} (t) \ge 0.   
\end{gather*} 
Then, we have that: conditionally on the event $\bigl\{ \omega \in \Omega \, |  \, \back{X}{T} \sim P_{\mathrm{data}}^{[j]} \bigr\}$,  
\begin{align*}
\limsup_{D \to \infty} \tfrac{1}{D} \log W_D^{[i]} (t, \back{X}{t}) \le 
- \, \widetilde{C}_\infty^{\,[i, j]} (t), \qquad \forall i \in [K],  
\end{align*}
%
holds true almost surely for each $t \in [0, T)$.   
\end{proposition}
Proposition \ref{prop:LDP_posterior} implies that in the high ambient dimensional limit $D \to \infty$, the decay of the posterior weight is mainly governed by the KL divergence between pairs of noised cluster distributions. The data geometry stated in Assumptions \ref{assump:prior}, \ref{assump:cov} and \ref{assump:mean} ensure that the KL divergence (signal) eventually dominates the fluctuation terms (noise) specified in (\ref{eq:C}) as $D \to \infty$. If the clusters in the target distribution are well-separated, e.g., in the sense of the first or second moment, then the signal (KL divergence) is also expected to increase as $t$ approaches the terminal time $T$. 
That is, for small $t \in [0, T-\varepsilon]$, the noised clusters are nearly indistinguishable, but the gap gradually increases as the denoising proceeds, where those differences are measured by the KL divergence. In the early stage of denoising, the posterior weights are basically similar to the prior weights, but, for some classes, these posterior weights start to degenerate toward 0, followed by a rapid posterior concentration for a single cluster $j$. 
In particular, such a rapid cluster commitment occurs with an exponential rate after the signal $\KL{\for{p}{T-t}^{[j]}}{\for{p}{T-t}^{[i]}}$ reaches a certain level so that $- \KL{\for{p}{T-t}^{[j]}}{\for{p}{T-t}^{[i]}} + o(D)$ takes a sufficiently large negative value. 
\\ 

\noindent 
\textit{Proof of Proposition \ref{prop:LDP_posterior}.} 
Set $\delta = D^{-q}, \, q \ge 2$. Recall the definition of the event $A^{[i,j]}$ in (\ref{eq:set_A}) and define 
$ E_D^{[j]} (t, \delta) = \bigcap_{i \in [K]} A^{[i,j]} \bigl(t, \tfrac{\delta}{K} \bigr) 
$
and 
$B^{[j]} = \{ \omega \in \Omega \, | \, \back{X}{T} \sim P_{\mathrm{data}}^{[j]} \}$.  
Then, Proposition \ref{prop:NA_weight_bd} yields
\begin{align*}
\sum_{D = 1}^\infty 
\mathbb{P} \Bigl( \bigl( E_D^{[j]} (t, D^{-q}) \bigr)^c \, | \,  B^{[j]} \Bigr) \le \sum_{D = 1}^\infty D^{-q} < \infty. 
\end{align*}
Thus, the first Borel-Cantelli lemma gives that: 
\begin{align*}
\mathbb{P} \Bigl( \liminf_{D \to \infty} E_D^{[j]} (t, D^{-q})  | B^{[j]} \Bigr) = 1, 
\end{align*} 
that is, for $\mathbb{P}(\cdot | B^{[j]})$ almost every $\omega$, there exists $D^\dagger \in \mathbb{N}$ such that for all $D \ge D^\dagger$ we have 
\begin{align} \label{eq:BC_ineq}
\log W_D^{[i]} (t, \back{X}{t}) +  C_D^{[i,j]} (t, \tfrac{\delta}{K}) \le 0 
\end{align}
for all $i \in [K]$. Dividing the both sides of (\ref{eq:BC_ineq}) by $D$ and taking $\limsup_{D \to \infty}$ gives 
$$\limsup_{D \to \infty} \tfrac{1}{D} \log W_D^{[i]} (t, \back{X}{t}) \le - \limsup_{D \to \infty} \tfrac{1}{D} C_D^{[i,j]} (t, \tfrac{\delta}{K}).$$ 
The proof completes once we show that for $\delta =D^{-q}, \, q \ge 2$, 
\begin{align} \label{eq:C_scale}
\limsup_{D \to \infty} \tfrac{1}{D} C_D^{[i,j]} (t, \tfrac{\delta}{K}) = \widetilde{C}_\infty^{\, [i, j]} (t),   
\end{align} 
where all the fluctuation terms in \eqref{eq:C} multiplied by $1/D$ converges to $0$ as $D \to \infty$. To this end, we exploit the following upper bounds for the fluctuation terms in (\ref{eq:C}) whose proof is contained in Appendix \ref{sec:pf_bd_vars}: 
\begin{lemma} \label{lemma:bd_vars}  
It holds that for each $t \in [0, T)$ and all $i,j \in [K]$,  
\begin{align} 
\mathscr{T}_D^{[i,j]} (t) & \le \bigl(1 + \snr{T-t} \cdot \| \Sigma_i - \Sigma_j \|_2 \bigr) \cdot m_{T-t}^2 \| \mu_i - \mu_j \|^2_{(S_{T-t}^j)^{-1}}; \label{eq:T_bd}\\ 
\mathscr{U}_D^{[i,j]} (t) & \le \tfrac{1}{2} \snr{T-t}^2 \bigl\| \Sigma_{i} - \Sigma_j \bigr\|_F^2 
; \label{eq:U_bd} \\ 
\mathscr{V}_D^{[i,j]} (t) & \le \snr{T-t} \cdot \| \Sigma_i - \Sigma_j \|_2.  \label{eq:V_bd}  
\end{align}
\end{lemma}
Lemma \ref{lemma:bd_vars} with Assumptions \ref{assump:prior} -- 
\ref{assump:mean} yields that: for each $t \in [0, T)$ with a fixed $T \ge 1$, $\mathscr{T}_D^{[i,j]} (t)= \mathcal{O} (D^{\alpha +1})$, $\mathscr{U}_D^{[i,j]} (t) = \mathcal{O} (D^{2\alpha})$ and $\mathscr{V}_D^{[i,j]} (t) = \mathcal{O} (D^{\alpha})$,  
%
%
%
leading to the desired equation (\ref{eq:C_scale}), where we used $\| \Sigma_i - \Sigma_j \|_F^2 = \mathcal{O} (D^{2\alpha})$ shown in \eqref{eq:Sigma_diff_F}. The proof of Proposition \ref{prop:LDP_posterior} is now complete.

\subsection{Posterior Concentration and the Scale of SNR} \label{sec:posterior_3} 
We hereby examine the time scale at which the posterior weight decay, or concentration, occurs. Inspired by the high-dimensional asymptotics presented in Proposition \ref{prop:LDP_posterior}, 
we analytically identify the point in time after which the main signal consistently outweighs other components in equation (\ref{eq:C}), such as the ratio of prior weights and fluctuation terms. Specifically, by setting a fixed early stopping criterion $\varepsilon \in (0,1)$, we consider the time $t^\dagger \in [0, T - \varepsilon]$ such that $C_D^{[i,j]}(t, \delta/K) \geq \eta \KL{\for{p}{T-t}^{[j]}}{\for{p}{T-t}^{[i]}}$ for all $t \geq t^\dagger$, for a fixed constant $\eta \in (0,1)$. To analytically characterise the scale of $t^\dagger$, we introduce an additional condition: 
\begin{assump} \label{assump:non-concentration}
There exists a constant $\Csix > 0$ such that for all $i, j \in [K]$, 
\begin{align} \label{eq:mean_nonconcentration}
\max_{k \in [D]} \bigl| \langle \mu_i - \mu_j, u_k (j)  \rangle \bigr| \le \Csix, 
\end{align}
where $u_k (j)$ is the eigenvector associated with the eigenvalue $\lambda_k (\Sigma_j)$.
\end{assump}
Assumption \ref{assump:non-concentration} states that the projection of the mean difference vector between clusters onto the eigenvectors of one cluster does not concentrate on a few principal components. We then have the following result, whose proof is postponed to Section \ref{sec:pf_first_main}:  

\begin{theorem}[SNR Scale for Posterior Concentration]
\label{thm:NA_transition_est}
Let $T \ge 1$ and $\varepsilon \in (0,1)$ be a fixed early stopping. Also set $\eta \in (0,1)$ and $q \ge 2$. Let Assumptions \ref{assump:prior}, \ref{assump:cov}, \ref{assump:mean}, \ref{assump:non-concentration} hold. Then, we have the following: 
There exists $D^\dagger = D^\dagger (\eta, q, \varepsilon)$ such that for any integer $D \ge D^\dagger$,  
\begin{align} \label{eq:key_hpb}
\mathbb{P} \biggl( \bigcap_{\substack{i \in [K]}} \Bigl\{ W_D^{[i]} (t, \back{X}{t}) \le  
\exp \Bigl( - \eta \KL{\for{p}{T-t}^{[j]}}{\for{p}{T-t}^{[i]}} \Bigr) \Bigr\} 
\, | \, \back{X}{T} \, \sim P_{\mathrm{data}}^{[j]} \biggr) \ge 1 - D^{-q}, \quad j \in [K],   
\end{align} 
for each $t \in \Bigl[ \max \bigl\{0, T + \tfrac{1}{2} \log \tfrac{\underline{S}}{1 + \underline{S}} \bigr\}, T-\varepsilon \Bigr]$, equivalently $\snr{T-t} \in \Bigl[\max \bigl\{\snr{T}, \underline{S} \bigr\}, \snr{\varepsilon} \Bigr]$, where $\underline{S} \equiv \underline{S} (D, q, \eta)$ satisfies:
\begin{align} \label{eq:threshold_scale}
\underline{S} = \Theta \left( \tfrac{\log (KD)}{D} \right). 
\end{align}
Besides, for any $\beta \in (0,1)$, 
\begin{align*}
\mathbb{P} \biggl( \bigcap_{\substack{i \in [K] \\  i \neq j}} \bigl\{ W_D^{[i]} (t, \back{X}{t}) \le  \beta \bigr\} 
\, | \, \back{X}{T} \, \sim P_{\mathrm{data}}^{[j]} \biggr) \ge 1 - D^{-q}, \quad j \in [K],    
\end{align*} 
whenever the reverse time $t$ satisfies 
\begin{align} \label{eq:snr_transition}
\snr{T-t} \ge L \cdot \frac{\max \bigl\{ \log (KD), \, \log \tfrac{1}{\beta} \bigr\}}{D} 
\end{align}
 for some constant $L > 0$ independent of $D$.  
\end{theorem}
Theorem \ref{thm:NA_transition_est} establishes an \emph{SNR scaling law for posterior concentration}. That is, for trajectories eventually sampling from cluster $j$, the posterior weights assigned to all other classes are sufficiently small with high probability once the SNR reaches a threshold $\Theta \bigl( \log (KD) / D \bigr)$. In the proof, Assumption \ref{assump:mean} (mean separation ) together with Assumption \ref{assump:non-concentration} plays a role in quantifying the signal $\KL{\for{p}{T-t}^{[j]}}{\for{p}{T-t}^{[i]}}$, leading to the obtained scale of the SNR threshold $\underline{S}$. 

The above result also gives rise to the following key implications on the appropriate scaling of $T$ with respect to the ambient dimension $D$ and its effect on the length of the mixing phase:
\begin{corollary} \label{cor:T_scale}
Consider the same setting as the statement in Theorem \ref{thm:NA_transition_est} and let $t^\dagger =  \max \bigl\{0, T + \tfrac{1}{2} \log \tfrac{\underline{S}}{1 + \underline{S}} \bigr\}$. 
Then, we have the following: 
\item[(i).] If the time horizon $T >0$ is a fixed constant, then 
    $t^\dagger \to 0$ as $D \to \infty$, i.e., \eqref{eq:key_hpb} holds for each $t \in [0, T-\varepsilon]$. 
\item[(ii).] If $T \ge \tfrac{1}{2} \log D$, then $t^\dagger \gtrsim \log \log D.$ Thus, the threshold $t^\dagger$ does not collapse to $0$ as $D \to \infty$, i.e., the mixing phase can remain non-trivial. 
\end{corollary}
We remark that for $T \ge \tfrac{1}{2} \log D$, we have $\snr{T} \le 2/D, \, D \ge 2$, and so $\snr{T} \le \underbar{S}$ for any sufficiently large $D$. The choice of time interval $T \ge \tfrac{1}{2} \log D$ is also reasonable in reducing the gap between $\for{p}{T}$ and $\mathcal{N} (0, I_D)$ so that the forward noising distribution at time $T$ is sufficiently close to the normal distribution. See also Theorem \ref{thm:main} below. \\

\noindent 
\textit{Proof of Corollary \ref{cor:T_scale}.} The first statement (i) is immediate by noticing that $T + \tfrac{1}{2} \log \tfrac{\underline{S}}{1 + \underline{S}} \to - \infty$ as $D \to \infty$ given \eqref{eq:threshold_scale}. For (ii), $T \ge \tfrac{1}{2} \log D$ with (\ref{eq:threshold_scale}) gives $T + \tfrac{1}{2} \log \tfrac{\underline{S}}{1 + \underline{S}} \ge \tfrac{1}{2} \log \tfrac{D \underline{S}}{1 + \underline{S}} \gtrsim \log \log D$ for any sufficiently large $D$. The proof is now complete.  
\subsection{KL Convergence Analysis -- Intrinsic Dimension Adaptivity via Posterior Concentration}
\label{sec:KL}
Given the above arguments on the dynamical regimes, we now argue how the time-discretisation error scales with the ambient dimension $D$ and the intrinsic dimensionality $M_i$ involved in each cluster specified in Assumption \ref{assump:cov}. 
%
%
%
%
%

Let $\varepsilon \in (0,1)$ be a fixed early stopping. We write the time-discretisation schedule $\{ t_i \}_{i \in [N]}$ satisfying $0 = t_0 < t_1 < \cdots < t_N = T - \varepsilon$ with $T \ge 1$ and the number of steps $N \in \mathbb{N}$. We define the time-discretised denoising sampler : 
\begin{align*}
d \back{Y}{t} = \Bigl(\back{Y}{t} - 2  \tfrac{\back{Y}{t}}{\sigma_{T-t}^2} + 2 
\tfrac{m_{T-t}}{\sigma_{T-t}^2} \hat{\mu} ({T-t_i}, \back{Y}{t_i})  \Bigr) dt + \sqrt{2} d {B}_t, \qquad t \in [t_i, t_{i+1}],   
\end{align*} 
with $\back{Y}{0} \sim \mathcal{N} (0, I_D)$, where $\hat{\mu}_{T-t_i} (y)$ is an estimator of the true posterior mean $\mu_{T-t_i} (y) = \mathbb{E} [\for{X}{0} | \for{X}{T - t_i} = y] $. We write $\for{\mathbb{P}}{t}$ and $\back{\mathbb{Q}}{t_i}$ as the law of $\for{X}{t}$ and $\back{Y}{t_i}$, respectively and study the upper bound for $\KL{\for{\mathbb{P}}{\varepsilon}}{\back{\mathbb{Q}}{t_N}}$. 
%
%
We then need to decide the schedule for the time discretisation. In this work, we adopt the following clear and simple update rule for grid points: for a precision parameter $\tau \in (0,1/4]$, 
\begin{align}
\begin{aligned} \label{eq:t_schedule}
t_ 0  & = 0; \\ 
t_{j+1} & = \min \bigl\{ t_j + \tau \sigma_{T-t_{j}}^2, \, T - \varepsilon \bigr\}, \qquad j \in \{0, \ldots, N-1\}  
\end{aligned} 
\end{align}
%
where $N$ is the minimum integer s.t. $t_{N-1} + \tau \sigma_{T-t_{N-1}}^2 \ge T - \varepsilon$. 
The design of the schedule is determined by the local noise $\sigma_{T-t_j}^2$, 
but is independent of the geometry of multimodal distributions. Let $h_{j+1} = t_{j+1} - t_j, \, j = 0, 1, \ldots, N-1$, then this schedule has two properties. First, we have $h_{j+1} \le 2 \tau \min \{1, T - t_j\}$, which aligns with the exponential-decreasing-type schedules often employed in convergence analyses; see, e.g., \cite{ben:24, pota:25, Huang:26}. The second property is that the number of time-steps $N$ increases as $1/\tau$, as shown in the following lemma, whose proof is postponed to Appendix \ref{sec:pf_N_scale}: 
\begin{lemma} \label{lemma:complexity_N}
For a given precision parameter $\tau \in (0, 1/4]$, the number of steps $N$ specified in the time-discretisation schedule (\ref{eq:t_schedule}) satisfies 
\begin{align} \label{eq:N_vs_tau} 
N
= \Theta \Bigl(1 +  \tfrac{1}{\tau} 
\bigl( T - 1 + \log \tfrac{1}{\varepsilon} \bigr)  \Bigr).  
\end{align}
\end{lemma} 
%

To state our main result, we introduce the following assumption regarding the accuracy of score estimation. 
\begin{assump}[Score Estimate Accuracy] \label{assump:score}
There exists $\varepsilon_{\mathrm{score}} > 0$ such that
\begin{align*}
\frac{1}{T} \sum_{j = 0}^{N-1} h_{j+1} \times 
\mathbb{E} \bigl[ \| \nabla \log p_{T-t_j} (\back{X}{t_j}) 
- \hat{s} (T-t_j, \back{X}{t_j}) \|^2   \bigr] 
\le \varepsilon_{\mathrm{score}}^2, 
\end{align*}
where 
\begin{align} \label{eq:tweedie}
\hat{s} (T-t_j, x) \equiv - \frac{x - m_{T-t_j} \hat{\mu}{(T-t_j, x)}}{\sigma_{T-t_j}^2}, \qquad x \in \mathbb{R}^D.  
\end{align} 
\end{assump} 
Treating the accuracy as a black box, as per this assumption, is commonly adopted in diffusion model convergence analysis \citep{ben:24, pota:25, Huang:26}.    
We now have the following KL convergence guarantee: 
\begin{theorem}[KL Convergence under Dynamical Regimes] \label{thm:main}
Let $T \ge \max \bigl\{1, \tfrac{1}{2} \log D \bigr\}$ and $\varepsilon \in (0,1)$ be a fixed early stopping parameter. Consider the time-discretisation schedule (\ref{eq:t_schedule}) with $\tau \in (0, 1/4]$. Let Assumptions \ref{assump:prior}, \ref{assump:cov}, \ref{assump:mean}, \ref{assump:non-concentration} and  \ref{assump:score} hold. Also, assume $\mathbb{E}_{\data} [\| X \|^2] < \infty$, $K \lesssim D^\gamma$ for a fixed $\gamma > 0$ and $\varepsilon \gtrsim R / T$, where $R \in [0, 1)$ is found in Assumption \ref{assump:cov}. 
Set $M^\dagger = \max_{k \in [K]} M_k$. 
Then, there exists $D^\dagger \in \mathbb{N}$ such that for all $D \ge D^\dagger$, we have $\kl  (\for{\mathbb{P}}{\varepsilon} || \back{\mathbb{Q}}{t_N}) \lesssim \mathcal{E}_1 + \mathcal{E}_2 + \mathcal{E}_3$, where 
\begin{align} 
\mathcal{E}_1 \equiv \tau  M^\dagger \Bigl( T  + \log \tfrac{1}{\varepsilon}\Bigr), \qquad 
\mathcal{E}_2 \equiv T \varepsilon_{\mathrm{score}}^2, \qquad   
\mathcal{E}_3 \equiv \bigl( D + \mathbb{E}_{\data} [\| X \|^2] \bigr) e^{-2T}. \label{eq:main}   
\end{align}
%
\end{theorem} 
The proof of Theorem \ref{thm:main} is postponed to Section \ref{sec:pf_main}. The bound \eqref{eq:main} shows the three sources of errors -- $\mathcal{E}_1$: time-discretisation error; $\mathcal{E}_2$: score-estimation error; and $\mathcal{E}_3$: prior gap -- which aligns well with the standard theoretical results in the literature. The main highlight is that the discretisation error scales with the intrinsic dimension $M^\dagger$, 
rather than the ambient dimension $D$, despite the $\mathcal{O}(D)$ inter-class separation, while the number of mixture components $K$ is allowed to grow with $D$ polynomially. 
Moreover, combining (\ref{eq:N_vs_tau}) with Theorem \ref{thm:main} gives the following complexity bound:
\begin{corollary}
Consider the setting of Theorem \ref{thm:main} and set $A \equiv M^\dagger \bigl( T  + \log \tfrac{1}{\varepsilon}\bigr)$. For any $\chi \in \bigl(0, A/4 \bigr]$, choosing $\tau = \chi / A \in (0, 1/4]$ leads to $\mathcal{E}_1 = \chi$, where the number of steps $N$ satisfies 
\begin{align*}
N = \mathcal{O} \left( \tfrac{ M^\dagger}{\chi} \Bigl( T + \log \tfrac{1}{\varepsilon} \Bigr)^2 \right).   
\end{align*}     
\end{corollary}
\begin{remark}
The requirement for the early stopping criterion $\varepsilon \gtrsim R / T$ gives insight into the role of $\varepsilon$ and its relation to the geometry of the data. When $\nu^\dagger = \max_{k \in [K]} \nu_k$ is exactly 0 (each cluster's covariance is exactly low-rank) or nearly 0, i.e., data within clusters concentrates on each thin subspace, then the corresponding $R$ is exactly or nearly 0, and one can choose a very small $\varepsilon$.
If $R$ is large, i.e., clusters no longer have a low-dimensional structure, then one instead obtains an error bound depending on $D$ linearly without relying on such a requirement for $\varepsilon$. 
\end{remark}
\subsection{Overview of the Proof of Theorem \ref{thm:main}} 
\label{sec:poc}

We briefly outline the main arguments for the proof of Theorem \ref{thm:main} and highlight how the phase transition analysis is employed to obtain the intrinsic-dimension ($M^\dagger$) scaling in the time-discretisation error, even when the number of clusters $K$ grows polynomially with $D$. 

The proof consists of five main steps. The first step in Section \ref{sec:pf_step_1} uses standard arguments \citep{ben:24, conf:25, Huang:26} to bound $\kl  (\for{\mathbb{P}}{\varepsilon} || \back{\mathbb{Q}}{t_N})$ as the sum of three terms that correspond to the error from approximating the score, the error in how we initiate the denoiser and the time-discretisation error. The main novelty in our proof is how we bound the time-discretisation error, which is covered in Steps 2 to 4 and outlined below. Step 5 then collects together all the error terms, and obtains the simplified bound presented in \eqref{eq:main}.  

In Step 2 (see Section \ref{sec:pf_step_2}) we provide an upper bound on the time-discretisation error involving the following term: 
$$
\tau \sum_{j = 1}^{N-1} \Bigl( \snr{T-t_{j+1}}  - \snr{T-t_j} \Bigr) \cdot \mathbb{E} \Bigl[ \mathrm{tr} \bigl(\Sigma_{T - t_j} (\for{X}{T-t_j}) \bigr) \Bigr],
$$
where $\Sigma_{T-t} (y) = \mathrm{Cov} [\for{X}{0} | \for{X}{T-t} = y], \, (t, y) \in [0, T) \times \mathbb{R}^D$, is the posterior covariance. Under the GMM framework, the posterior covariance can be decomposed as shown by the following lemma.
\begin{lemma} \label{lemma:PC}
It holds that for all $(t, y) \in [0, T) \times \mathbb{R}^D$, $\Sigma_{T-t} (y) = \Sigma_{T-t}^{I} (y) + \Sigma_{T-t}^{II} (y)$ with 
\begin{align}
\begin{aligned}  \label{eq:posterior_cov}
\Sigma_{T-t}^I (y)   
& = \sum_{k= 1}^K  W_D^{[k]} (t, y) \sigma_{T-t}^2 \Sigma_k (S_{T-t}^k)^{-1}; \\[0.1cm]
\Sigma_{T-t}^{II} (y) 
& = \frac{1}{2} \sum_{j,k \in [K]} W_D^{[j]} (t, y) W_D^{[k]} (t, y) \bigl( \mu^{[j]} (t, y) - \mu^{[k]} (t, y) \bigr)^{\otimes 2}, 
\end{aligned}
\end{align} 
where we have set 
\begin{align} \label{eq:posterior_mean_cluster}
\mu^{[j]} (t, y) = \mathbb{E} [ \for{X}{0} \, | \, \for{X}{T-t} = y, \, \for{X}{0} \sim P_{\mathrm{data}}^{[j]}] = \frac{\sigma_{T-t}^2}{m_{T-t}} \nabla \log \for{p}{T-t}^{[j]} (y)  + \frac{y}{m_{T-t}}. 
\end{align}
\end{lemma}  
The proof of Lemma \ref{lemma:PC} is provided in Appendix \ref{sec:pf_PC}. The term $\Sigma_{T-t}^I$ captures the local within-cluster geometry, whereas $\Sigma_{T-t}^{II}$ reflects the between-cluster geometry and may therefore scale with $D$ through differences between components.

In Section \ref{sec:pf_step_3} (Step 3), we show the upper bound on the term $\mathbb{E} \bigl[ \mathrm{tr} \bigl(\Sigma_{T - t_j}^I (\for{X}{T-t_j}) \bigr) \bigr]$ depends linearly on $M^\dagger$ but is independent of $K$. Section \ref{sec:pf_step_4} (Step 4) deals with the inter-class term $\mathbb{E} \bigl[ \mathrm{tr} \bigl(\Sigma_{T - t_j}^{II} (\for{X}{T-t_j}) \bigr) \bigr]$ by exploiting the phase transition. Roughly, we divide the time interval $[0, T-\varepsilon]$ into two regimes according to whether $\snr{T-t_j} \le L \bigl( \log (KD) \bigr) / D$ for an arbitrary fixed constant $L > 0$.  
During the first regime (mixing phase), the small SNR offsets the mismatches in the first and second moments that can grow polynomially with $D$. For the latter regime (cluster commitment phase), Theorem \ref{thm:NA_transition_est} implies that, conditional on a trajectory terminating in a cluster $k$, at least one of the two posterior weights $W_D^{[i]} (t_j, \for{X}{T-t_j}), \, W_D^{[j]} (t_j, \for{X}{T-t_j}), \, i \neq j$, is exponentially small in $D$ with high probability, thereby suppressing the inter-class difference. Furthermore, the dependence on $K$ is shown to be logarithmic via a finite exponential-moment bound for the maximum over the Gaussian components. 

\section{Numerical Experiments of Posterior Commitments in Real Data}
\label{sec:numerical_experiments}
So far, we have studied the intrinsic dimension adaptivity of diffusion models for target distributions formulated as a union of distinct (linear) manifolds within a Gaussian mixture framework. The key ingredient for showing such adaptivity in the KL error analysis (Theorem \ref{thm:main}) is that the scale of the SNR for posterior weights concentration is approximately $\Theta (1/D)$, as shown in Theorem \ref{thm:NA_transition_est} under cluster separation. This section provides empirical evidence that this SNR scale at posterior concentration can also be observed for high-dimensional data beyond the GMM setting.  
All technical details of the implementation can be found in Appendix \ref{sec:details_NE}, and the code to reproduce the experiments is available at \url{https://github.com/YugaIgu/diffusion_models_phase_transition}

\subsection{Set-Up}  \label{sec:set-up}
We consider two types of high-dimensional data, which are typical examples of the application of diffusion models: (a) Single cell RNA sequencing (scRNA-seq) of peripheral blood mononuclear cells (PBMCs) with $32738$ genes and (b) Image data (STL-10) consisting of a resolution of 96 by 96 pixels with 3 colour channels (RGB). For dataset (a), we vary the ambient dimension $D$ by retaining the top $D \in \{500,1000,2000,4000,8000\}$ highly variable genes (HVGs). 
For the STL-10 image data, we construct lower-resolution versions to obtain three resolutions: $D = 32 \times 32 \times 3$, $D = 64 \times 64 \times 3$, and $D = 96 \times 96 \times 3$ (raw data), by using bilinear interpolation, implemented in PyTorch. 
Then, we select three classes (labels), specifically, $\mathcal{C}_{\mathrm{PBMC}} = \{\mathrm{CD4\,T}, \mathrm{CD14 \!+ \, Monocytes}, \mathrm{CD19\!+ B}\}$ and $\mathcal{C}_{\mathrm{image}} = \{\mathrm{dog}, \mathrm{cat}, \mathrm{airplane}\}$ for datasets (a) and (b). 

We simulate the denoising SDE to sample from the target clusters. We directly use the score of the noised empirical distribution rather than a trained neural-network score, in the spirit of the exact-empirical-score setting considered in \cite{BBDM:24}. This allows us to isolate the dynamical phase transition from the uncertainty due to score estimation error. 
We perform the following four steps to observe the empirical SNR at the posterior weight concentration: 
\begin{itemize}[leftmargin=0.3cm]
\item[I.] \textit{Simulating trajectories.} We set the terminal time as $T = 6$ and simulate $M = 500$ independent backward trajectories $\{\bar{X}_{t_j}^{[k]}\}_{1 \le j \le n, 1 \le k \le M}$ using the Euler-Maruyama discretisation with a uniform time-step $h = t_{j+1} - t_j = 0.003$ and so the number of steps $n = T/h = 2000$. These are used as numerical approximations for the exact backward process. 
\item[II.] \textit{Labelling.} For each trajectory, we assign a label of the class from which the path eventually samples, which is obtained as the argument of the maximum of the posterior weight (\ref{eq:posterior_weight}):  
$
C^{[k]} \equiv \mathrm{argmax}_{c \in \mathcal{C}} \, W_D^{[c]} (t_J, \bar{X}_{t_{J}}^{[k]}), \, k = 1, \ldots, M, 
$ 
where $J = n - 10$ (to avoid numerical instability) and $\mathcal{C}$ is $\mathcal{C}_{\mathrm{PBMC}}$ or $\mathcal{C}_{\mathrm{image}}$.
\item[III.] \textit{Sorting.} Those paths are sorted within the three distinct classes as
$
\mathbb{X}^c = \bigl\{ \bigl\{ \bar{X}_{t_j}^{[k]}\bigr\}_{j = 1, \ldots, J} \, | \, C^{[k]} = c, \,  k = 1, \ldots, M \bigr\}, \, c \in \mathcal{C},  
$
that is, the set of paths eventually sampling from the class $c$. 
%
\item[IV.] \textit{Recording.} We then record the values of three types of posterior weights $W^{[c]}_D$, $c \in \mathcal{C}$, evaluated with the trajectories contained in $\mathbb{X}^{c^\dagger}$ for some target class $c^\dagger \in \mathcal{C}$ at each time-step $j = 1, \ldots, J$. 
\end{itemize}    
In the above procedure, the prior weights are assumed to be equal across the classes, and the (intractable) posterior weights are estimated by plugging an empirical approximation of the forward noised density into the definition of the weight (\ref{eq:posterior_weight}). 

\begin{table}[t]
\centering
\footnotesize
\caption{
Empirical covariance geometry for PBMC and STL-10 across ambient dimensions.
For each class, we report the maximum covariance eigenvalue (\(\lambda_{\max}\)), the trace (\(\tr(\widehat\Sigma)\)),
and the 90\% effective rank (\(r_{\mathrm{eff},90}\)), computed from the empirical data.}
\label{tab:classwise_cov_geometry}
\renewcommand{\arraystretch}{1.10}

\begin{minipage}{0.48\linewidth}
\centering
\textbf{PBMC}
\vspace{0.3em}

\begin{tabular}{@{} llrrr @{}}
\toprule
Class & \(D\)
& \(\lambda_{\max}\)
& \(\tr(\widehat\Sigma)\)
& \(r_{\mathrm{eff},90}\) \\
\midrule
\multirow{5}{*}{CD4 T}
& \(500\)  & \(3.42\)  & \(469.88\)  & \(313\) \\
& \(1000\) & \(5.41\)  & \(943.59\)  & \(507\) \\
& \(2000\) & \(9.85\)  & \(1807.21\) & \(686\) \\
& \(4000\) & \(16.30\) & \(3600.29\) & \(817\) \\
& \(8000\) & \(25.72\) & \(7465.22\) & \(903\) \\
\addlinespace

\multirow{5}{*}{CD14+ Mono.}
& \(500\)  & \(8.41\)  & \(485.49\)  & \(209\) \\
& \(1000\) & \(14.14\) & \(985.57\)  & \(289\) \\
& \(2000\) & \(23.32\) & \(2057.65\) & \(339\) \\
& \(4000\) & \(36.64\) & \(4001.79\) & \(372\) \\
& \(8000\) & \(57.85\) & \(7781.14\) & \(395\) \\
\addlinespace

\multirow{5}{*}{CD19+ B}
& \(500\)  & \(9.03\)  & \(554.19\)  & \(182\) \\
& \(1000\) & \(13.03\) & \(1050.05\) & \(229\) \\
& \(2000\) & \(22.68\) & \(2177.58\) & \(254\) \\
& \(4000\) & \(37.03\) & \(4723.49\) & \(272\) \\
& \(8000\) & \(55.90\) & \(9230.85\) & \(285\) \\
\bottomrule
\end{tabular}
\end{minipage}
\hfill
\begin{minipage}{0.48\linewidth}
\centering
\textbf{STL-10}
\vspace{0.3em}

\begin{tabular}{@{} llrrr @{}}
\toprule
Class & \(D\)
& \(\lambda_{\max}\)
& \(\tr(\widehat\Sigma)\)
& \(r_{\mathrm{eff},90}\) \\
\midrule
\multirow{3}{*}{Dog}
& \(3072\)  & \(160.69\)  & \(783.73\)  & \(145\) \\
& \(12288\) & \(638.36\)  & \(3022.54\) & \(130\) \\
& \(27648\) & \(1437.94\) & \(7061.82\) & \(154\) \\
\addlinespace

\multirow{3}{*}{Cat}
& \(3072\)  & \(132.58\)  & \(700.49\)  & \(150\) \\
& \(12288\) & \(531.07\)  & \(2679.13\) & \(134\) \\
& \(27648\) & \(1195.06\) & \(6308.94\) & \(163\) \\
\addlinespace

\multirow{3}{*}{Airplane}
& \(3072\)  & \(193.60\)  & \(756.55\)  & \(105\) \\
& \(12288\) & \(778.14\)  & \(2920.58\) & \(97\) \\
& \(27648\) & \(1750.48\) & \(6823.12\) & \(119\) \\
\bottomrule
\end{tabular}
\end{minipage}
\end{table}

\begin{table}[t]
\centering
\footnotesize
\caption{
Pairwise data geometry of the STL-10 clusters across image resolutions. For each pair and ambient dimension $D$, we report the squared mean separation,
the maximum mean difference projection defined in \eqref{eq:mean_nonconcentration}, and the squared Frobenius
covariance separation, computed from empirical data.  
}
\label{tab:stl10_pairwise_geometry}
\renewcommand{\arraystretch}{1.12}
\begin{tabular}{@{} llrrr @{}}
\toprule
Pair & \(D\)
& \(\|\widehat\mu_i-\widehat\mu_j\|^2\)
& \(M_\mu\)
& \(\|\widehat\Sigma_i-\widehat\Sigma_j\|_F^2\) \\
\midrule

\multirow{3}{*}{Dog--Cat}
& \(3072\)  & \(9.20\)   & \(2.09\)  & \(4.03{\times}10^3\) \\
& \(12288\) & \(35.10\)  & \(4.10\)  & \(5.98{\times}10^4\) \\
& \(27648\) & \(80.46\)  & \(6.17\)  & \(3.19{\times}10^5\) \\
\addlinespace

\multirow{3}{*}{Cat--Airplane}
& \(3072\)  & \(232.68\)  & \(12.81\) & \(1.35{\times}10^4\) \\
& \(12288\) & \(928.48\)  & \(25.59\) & \(2.12{\times}10^5\) \\
& \(27648\) & \(2091.26\) & \(38.39\) & \(1.09{\times}10^6\) \\
\addlinespace

\multirow{3}{*}{Dog--Airplane}
& \(3072\)  & \(189.10\)  & \(10.74\) & \(1.18{\times}10^4\) \\
& \(12288\) & \(753.39\)  & \(21.46\) & \(1.85{\times}10^5\) \\
& \(27648\) & \(1697.45\) & \(32.19\) & \(9.55{\times}10^5\) \\
\bottomrule
\end{tabular}
\end{table}

\begin{figure}[h]
\centering
\includegraphics[width=1.00\textwidth, keepaspectratio]{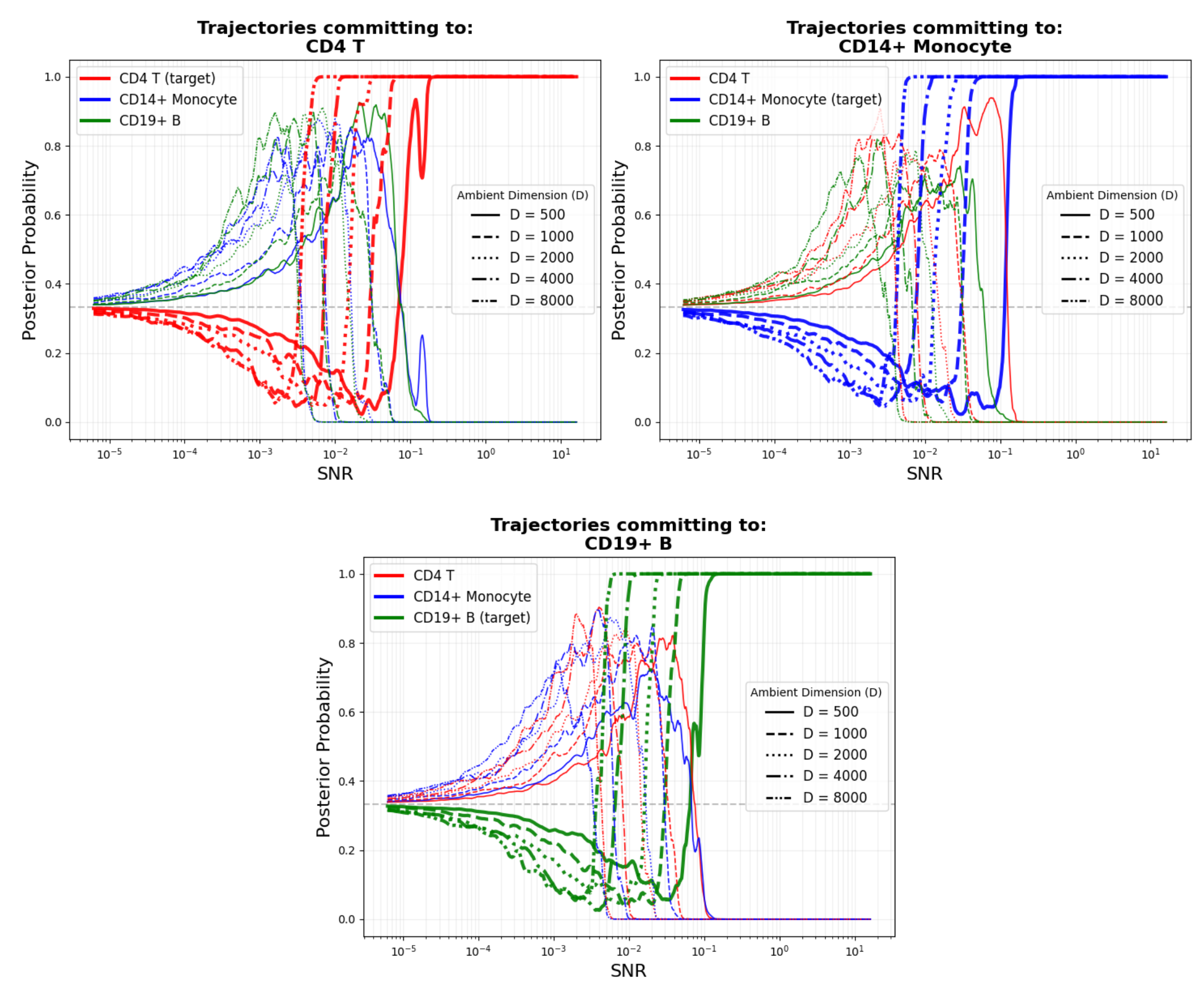}
\caption{
Posterior probability dynamics in the PBMC experiment. We construct increasing-dimensional representations by retaining the top-$D$ highly
variable genes, with $D\in\{500,1000,2000,4000,8000\}$. The posterior concentration timing shifts to lower SNR as $D$ increases. }
\label{fig:pbmc_phase_transition}
\end{figure}

\begin{figure}[h]
\centering
\includegraphics[width=1.0\linewidth]{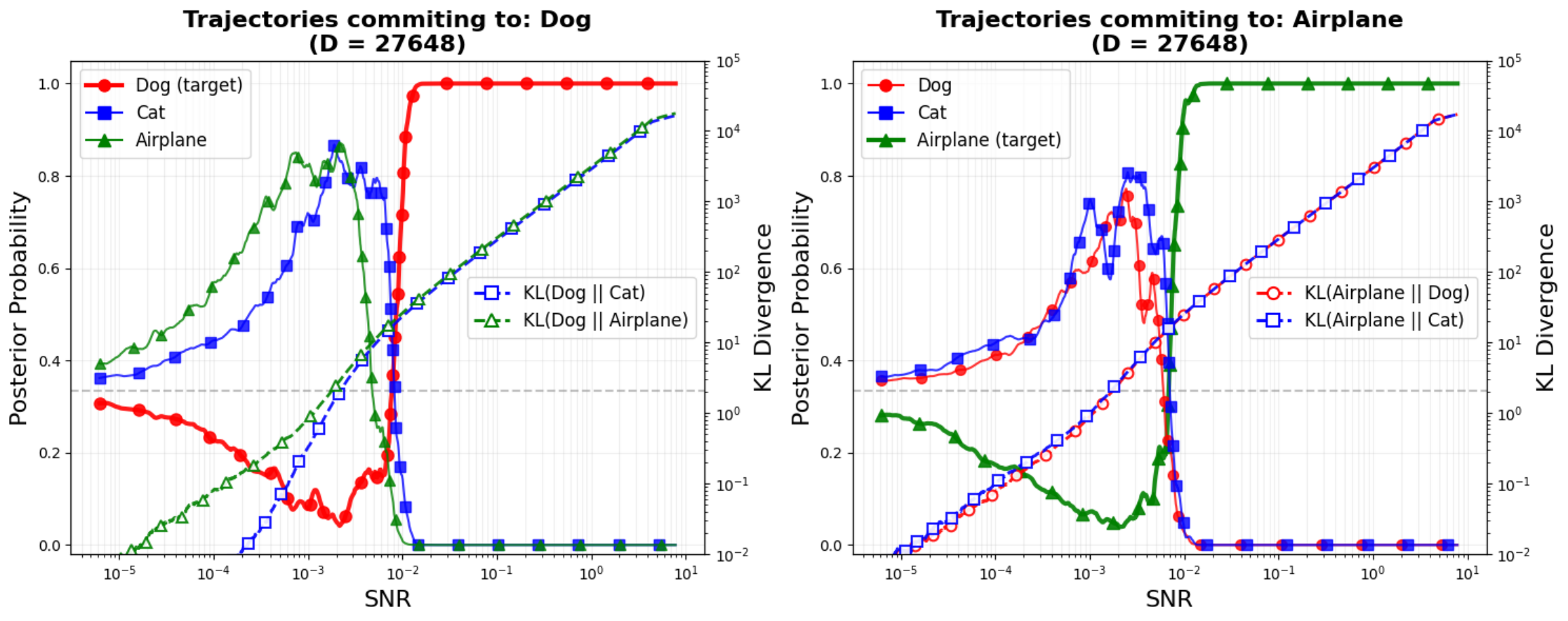}
\caption{
Posterior commitment in the STL-10 experiment. 
Posterior probabilities are evaluated using the reverse trajectories, grouped by their final committed class. 
Solid curves show the posterior probabilities of the three classes, while dashed curves show the corresponding noised KL divergences from the target class to competing classes. 
}
\label{fig:stl10_phase_transition}
\end{figure} 
\begin{figure}[h]
\centering
\begin{subfigure}[t]{0.45\textwidth}
\centering
\includegraphics[width=\linewidth, height=0.25\textheight,keepaspectratio]{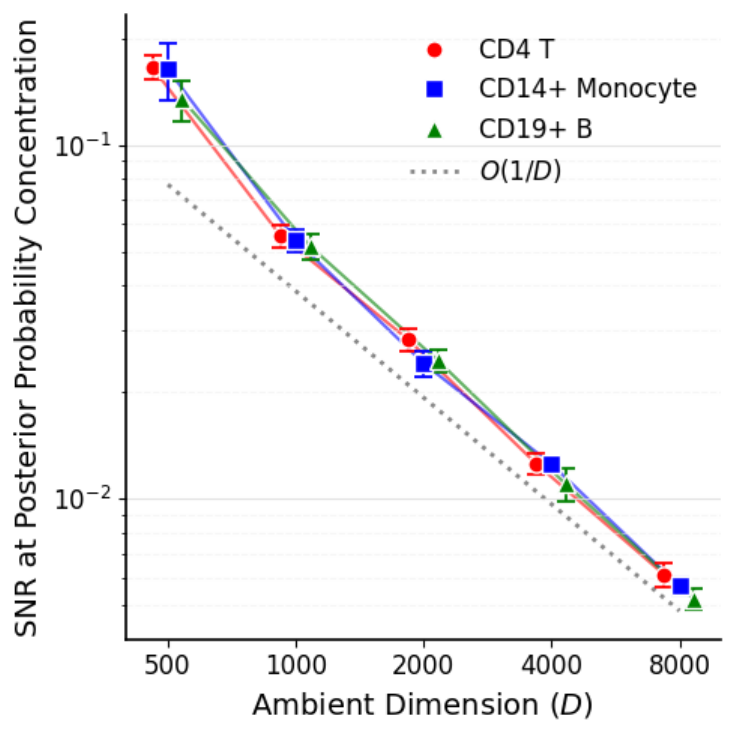}
\caption{PBMC single-cell data.}
\end{subfigure} 
\begin{subfigure}[t]{0.45\textwidth}
\centering
\includegraphics[width=\linewidth, height=0.25\textheight,keepaspectratio]{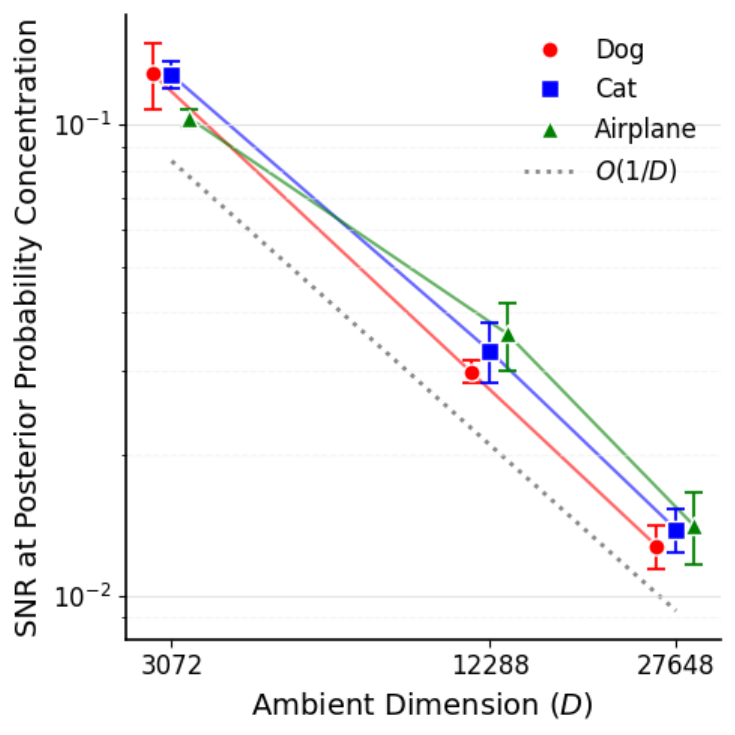}
\caption{STL-10 image data.}
\end{subfigure}
\caption{
Average empirical SNR at the posterior probability concentration over five repetitions of experiments, with the error bar showing their standard deviations. The plots show the SNR at which the lower 1 \% quantile of the posterior weight among
trajectories assigned to a given cluster permanently exceeds 0.99. In both STL-10 and PBMC experiments, the transition SNR exhibits an approximately $\mathcal{O}(1/D)$ decay. 
}
\label{fig:empirical_snr_scaling_law}
\end{figure}

\subsection{Results}
We first present the data geometry for the datasets (a) and (b) in Tables \ref{tab:classwise_cov_geometry} and \ref{tab:stl10_pairwise_geometry}. 
Table \ref{tab:classwise_cov_geometry} indicates that for the STL-10 datasets the 90\% effective rank is stable for varying ambient dimension $D$, while the maximum eigenvalue $\lambda_{\mathrm{max}}$ seems to grow approximately linearly with $D$. On the other hand, $\lambda_{\mathrm{max}}$ for PBMC data exhibits $\mathcal{O}(D^\alpha)$ growth with $\alpha < 1$. 
We observe from Table \ref{tab:stl10_pairwise_geometry} that the mean separation of STL-10 datasets scales linearly with $D$ and that covariance separations are very clear. 
The pairwise geometry for PBMC data also exhibits a similar separation growth as $D$ increases -- presented in Appendix \ref{sec:details_NE}. 

Figure \ref{fig:pbmc_phase_transition} demonstrates the dynamical concentration or decay of the posterior weights, evaluated for trajectories committing to a particular class, in the case of (a) PBMC data. Note that the reverse process proceeds from the left (low SNR) to the right (high SNR). The lines labelled with a target class $c^\dagger \in \mathcal{C}$ (`class name' (target) in the legend)  indicate the lower envelope of the posterior probability $W_D^{[c^\dagger]}$ achieved by 99\% of the paths in $\mathbb{X}^{c^\dagger}$, while the lines with competing classes $c \in \mathcal{C} \setminus \{c^\dagger\}$ indicate the upper envelope of $W_D^{[c]}$ by 99\% of the paths in $\mathbb{X}^{c^\dagger}$. 
As expected, for trajectories labelled with the chosen target cluster $c^\dagger$ (e.g., `CD4 T'), the posterior weight $W^{[c^\dagger]}_D$ eventually concentrates on $1$ as the SNR increases (as the physical time $t$ approaches the terminal $T$) while the others $W^{[c]}_D$ (e.g., `CD14+ Monocyte' and `CD19+ B') degenerate to $0$. More importantly, these plots showcase the rapid posterior concentration, or the phase transition, for each ambient dimension $D$, and its timing and corresponding SNR shifts earlier as $D$ increases. In particular, the sudden decay of non-target posterior weights, followed by the sharp rise of the target posterior weight, is interpreted as a sign that the KL signal between the corresponding noisy clusters begins to dominate the fluctuation terms in the log-posterior ratio.

In Figure \ref{fig:stl10_phase_transition}, we show the dynamical behaviour of the posterior probability for the image datasets with a fixed ambient dimension $D = 96 \times 96 \times 3$, but now with plots of the KL divergence between the noised target cluster $c^\dagger$ and the competing class $c \in \mathcal{C} \setminus \{c^\dagger\}$, precisely, $\KL{\for{p}{T-t_j}^{[c^\dagger]}}{\for{p}{T-t_j}^{[c]}}, \, j = 1, \ldots, n$. The KL divergence is not available in closed form and is therefore estimated from empirical data. 
The left panel of Figure \ref{fig:stl10_phase_transition} shows that, for trajectories eventually committing to the Dog class, the posterior weight of Airplane decays earlier than that of Cat. This behaviour is also reflected in the noised KL curves: the noised KL divergence between Dog and Airplane is already larger than that
between Dog and Cat in the early-SNR phase. This also agrees with the pairwise class separation presented in Table \ref{tab:stl10_pairwise_geometry}. Thus, the reverse trajectories committing to Dog first rule out the Airplane class, while the finer ambiguity between Dog and Cat persists until a larger SNR. By contrast, in the right panel of Figure \ref{fig:stl10_phase_transition}, for trajectories committing to Airplane, the KL curves against Dog and Cat are much closer to each other. Correspondingly, the Dog and Cat posterior weights decay in a more symmetric manner. 
These observations are consistent with the phase-transition analysis in Section~\ref{sec:phase_transition},
where the decay of non-target posterior weights is shown to be governed by the KL divergence between noised clusters.  

Finally, Figure~\ref{fig:empirical_snr_scaling_law} shows the mean and standard deviation of the first SNR level after which the lower 1\% quantile of the target posterior weight $W_D^{[c^\dagger]}$, evaluated among trajectories eventually assigned to the target cluster $c^\dagger$, remains above a threshold $0.99$ for all subsequent time steps. This indicates that for at least 99\% of the trajectories, the posterior weights permanently exceed the threshold after that time. The mean and standard deviation are computed over five independent repetitions of the four-step procedure explained in Section \ref{sec:set-up}. The metric provides an empirical proxy for the phase transition timing analysed in Section~\ref{sec:phase_transition}. 
In both the STL-10 and PBMC experiments, the SNR at the phase transition decreases as the ambient dimension $D$ increases, exhibiting an empirical $\mathcal{O}(1/D)$ scaling, in agreement with the posterior-commitment scaling predicted by the theory for GMMs (Theorem \ref{thm:NA_transition_est}) up to logarithmic factors.

\section{Proof of Theorem \ref{thm:main}} \label{sec:pf_main}
We write the measures of path $\back{X}{}$ and $\back{Y}{}$ over time interval $\mathbb{T} = [a, b], \, 0 < a  < b,$ as $\back{\mathbb{P}}{\mathbb{T}}$ and $\back{\mathbb{Q}}{\mathbb{T}}$, respectively, and $\snr{j} \equiv m_{T-t_j}^2 / \sigma_{T-t_j}^2, \, j = 0, 1, \ldots, N$, for ease of notation.  
\subsection{Step 1: Error Decomposition of KL Divergence} \label{sec:pf_step_1} We first decompose the KL divergence into three sources of error by following the arguments in convergence analysis for diffusion models \citep{ben:24, conf:25, Huang:26}. First, the data processing inequality and the equivalence of measures $\for{\mathbb{P}}{\varepsilon} \equiv \back{\mathbb{P}}{T-\varepsilon}$ lead to 
\begin{align*} 
\kl (\for{\mathbb{P}}{\varepsilon} || \back{\mathbb{Q}}{t_N}) \le \kl (\back{\mathbb{P}}{[0, T- \varepsilon]} || \back{\mathbb{Q}}{[0, T- \varepsilon]}). 
\end{align*} 
%
Further, employing the arguments in \cite[Section 3.3]{ben:24}, together with the use of Girsanov's theorem (see e.g., \cite[Appendix F]{ben:24} or \cite[Section 5.2]{chen:23-2}), we get $\KL{\back{\mathbb{P}}{[0, T-\varepsilon]}}{\back{\mathbb{Q}}{[0, T- \varepsilon]}} = R_1 + R_2$ with 
\begin{align*}
R_1 \equiv \KL{\for{\mathbb{P}}{T}}{\mathcal{N}(0, I_D)}, 
\quad 
R_2 \equiv \sum_{j = 0}^{N-1} \int_{t_j}^{t_{j+1}} \tfrac{m_{T-t}^2}{\sigma_{T-t}^4} \times 
\mathbb{E} 
\bigl[ 
\bigl\| {\mu}  (T-t, \back{X}{t})
- \hat{\mu} (T-t_j, \back{X}{t_j})  \bigr\|^2 \bigr] dt,   
\end{align*} 
where we recall $\mu  (t,x) = \mathbb{E} [\for{X}{0} | \for{X}{t} = x]$. For the first term $R_1$, following the proof arguments of \cite[Lemma 9]{chen:23} or \cite[Proposition 9]{ben:24}, we have: 
\begin{align} \label{eq:prior_gap}
R_1 \lesssim (D + \mathbb{E}_{\data} \bigl[ \| X \|^2 \bigr]) e^{-2T}. 
\end{align}
The second term $R_2$ is further  decomposed as $R_2 \lesssim R_{2}^I + R_{2}^{II}$, with  
\begin{align}
R_2^{I} & = \sum_{j = 0}^{N-1} \int_{t_j}^{t_{j+1}}  \tfrac{m_{T-t}^2}{\sigma_{T-t}^4} \times 
\mathbb{E} 
\bigl[ 
\bigl\| \hat{\mu}(T-t_j, \back{X}{t_j}) - \mu (T-t_j, \back{X}{t_j}) \bigr\|^2 \bigr] dt; \nonumber \\   
R_2^{II} & = \sum_{j = 0}^{N-1} \int_{t_j}^{t_{j+1}} 
\tfrac{m_{T-t}^2}{\sigma_{T-t}^4} \times \mathbb{E} 
\bigl[ 
\bigl\| \mu (T-t_j, \back{X}{t_j}) 
- \mu (T-t, \back{X}{t})  \bigr\|^2 \bigr] dt. \label{eq:e_disc}
\end{align} 
For the term $R_2^{I}$, we have from (\ref{eq:tweedie}), Tweedie's formula and Assumption \ref{assump:score} that
\begin{align} \label{eq:score_est_gap}
R_2^I \lesssim \sum_{j = 0}^{N-1} h_{j+1} \times  
\mathbb{E} 
\bigl[ 
\bigl\| \hat{s} (T-t_j, \back{X}{t_j}) - \nabla \log \for{p}{T-t_j} (\back{X}{t_j}) \bigr\|^2 \bigr] \le  T \varepsilon_{\mathrm{score}}^2. 
\end{align}
where we also used $\tfrac{m_{T-t}^2}{\sigma_{T-t}^4} \cdot \tfrac{\sigma_{T-t_j}^4}{m_{T-t_j}^2} \le \tfrac{m_{T-t_{j+1}}^2}{\sigma_{T-t_{j+1}}^4} \cdot \tfrac{\sigma_{T-t_j}^4}{m_{T-t_j}^2} \le 4 \exp (2 \tau)$ from Lemma \ref{lemma:noise_ratio_bd}.  

\subsection{Step 2. Decomposition of the Time-Discretisation Error} \label{sec:pf_step_2}
Our main focus is thus on bounding the term $R_2^{II}$ in \eqref{eq:e_disc}. We exploit the following result established in the literature. The proof can be found, e.g., in \cite[Lemma 9]{Huang:26}. 
\begin{lemma} \label{lemma:tr_diff}
Let $j \in [N]$. It then holds that for any $t \in [t_{j-1}, t_j]$, 
\begin{align*}
\begin{aligned} 
\mathbb{E} 
\Bigl[ 
\bigl\| \mu (T-t_{j-1}, \back{X}{t_{j-1}})  
- \mu (T-t, \back{X}{t})  \bigr\|^2 \Bigr]  
= \mathbb{E} \Bigl[ \mathrm{tr} \bigl(\Sigma_{T-t_{j-1}} (\back{X}{t_{j-1}}) \bigr) \Bigr]
- \mathbb{E} \Bigl[ \mathrm{tr} \bigl(\Sigma_{T-t} (\back{X}{t}) \bigr) \Bigr].  
\end{aligned} 
\end{align*}
\end{lemma}
Lemma \ref{lemma:tr_diff} implies that $t \mapsto \mathbb{E} \Bigl[ \mathrm{tr} \bigl(\Sigma_{T-t} (\back{X}{t}) \bigr) \Bigr]$ is a non-increasing function, thus 
\begin{align*}
\begin{aligned} 
\mathbb{E} 
\Bigl[ 
\bigl\| \mu (T-t_{j-1}, \back{X}{t_{j-1}})  
- \mu (T-t, \back{X}{t})  \bigr\|^2 \Bigr] 
\le \mathbb{E} \Bigl[ \mathrm{tr} \bigl(\Sigma_{T-t_{j-1}} (\back{X}{t_{j-1}}) \bigr) \Bigr]
- \mathbb{E} \Bigl[ \mathrm{tr} \bigl(\Sigma_{T-t_j} (\back{X}{t_j}) \bigr) \Bigr],  
\end{aligned} 
\end{align*} 
for any $t \in [t_{j-1}, t_j]$.  Recall $h_{j+1} \le \tau \sigma_{T-t_j}^2$ where the equality holds for $j = 0, 1, \ldots, N-2$. Then, we apply Lemma \ref{lemma:tr_diff} to the term $R_2^{II}$ and obtain  
\begin{align}
R_2^{II} & 
\le \sum_{j = 0}^{N-1} h_{j+1}  \tfrac{m_{T-t_{j+1}}^2}{\sigma_{T-t_{j+1}}^4}  \cdot
\Bigl( \mathbb{E} \bigl[ \mathrm{tr} \bigl(\Sigma_{T-t_j} (\back{X}{t_{j}}) \bigr) \bigr]
- \mathbb{E} \bigl[ \mathrm{tr} \bigl(\Sigma_{T-t_{j+1}} (\back{X}{t_{j+1}}) \bigr) \bigr] \Bigr) \nonumber \\
& \le 2 \tau \sum_{j = 0}^{N-1} \snr{j+1}
\Bigl( \mathbb{E} \bigl[ \mathrm{tr} \bigl(\Sigma_{T-t_j} (\back{X}{t_{j}}) \bigr) \bigr]
- \mathbb{E} \bigl[ \mathrm{tr} \bigl(\Sigma_{T-t_{j+1}} (\back{X}{t_{j+1}}) \bigr) \bigr] \Bigr) \nonumber \\ 
& \lesssim \tau \, \snr{1} \cdot \mathbb{E} \bigl[ \mathrm{tr} \bigl(\Sigma_{T} (\for{X}{T}) \bigr) \bigr] + \tau \sum_{j = 1}^{N-1} 
\bigl( \snr{j+1}  - \snr{j} \bigr) \cdot \mathbb{E} \bigl[ \mathrm{tr} \bigl(\Sigma_{T - t_j} (\for{X}{T-t_j}) \bigr) \bigr] \equiv \mathcal{T},   \label{eq:R_2_II}
\end{align}
where in the second line we used the bound $\sigma_{T-t_j}^2/\sigma_{T-t_{j+1}}^2 \le 2$, which is presented in Lemma \ref{lemma:noise_ratio_bd}. For the ease of notation, we write: 
\begin{align}
    \Delta_j = (\snr{j+1} - \snr{j}), \quad Z_j = (t_j, \for{X}{T-t_j}), \qquad j = 0, \ldots, N-1. 
\end{align}
Using the closed-form expression of posterior covariance (\ref{eq:posterior_cov}) in Lemma \ref{lemma:PC}, we decompose the term $\mathcal{T}$ as $\mathcal{T} = \mathcal{T}_1 + \mathcal{T}_2$, where 
%
\begin{align}
& \begin{aligned}
\mathcal{T}_1 =
\tau \sum_{j = 0}^{N-1} \Bigl(\snr{1} \mathbf{1}_{j = 0} + \Delta_j \mathbf{1}_{j \ge 1}  \Bigr) \sum_{k = 1}^K \mathbb{E} \bigl[ W_D^{[k]} (Z_j) \bigr] \cdot \sigma_{T-t_j}^2 \mathrm{tr} \Bigl( \Sigma_k \bigl( m_{T-t_j}^2 \Sigma_k + \sigma_{T-t_j}^2 I_D \bigr)^{-1} \Bigr); \label{eq:T_1}  
\end{aligned} \\[0.3cm] 
&\begin{aligned}
\mathcal{T}_2
= \tfrac{\tau}{2} \sum_{j = 0}^{N-1} \Bigl(\snr{1} \mathbf{1}_{j = 0} + \Delta_j \mathbf{1}_{j \ge 1}  \Bigr)  \sum_{\substack{k_1, k_2 = 1 \\ k_1 \neq k_2}}^K  
\mathbb{E} \Bigl[ W_D^{[k_1]} (Z_j) W_D^{[k_2]} (Z_j) \bigl\| \mu^{[k_1]} (Z_j)
- \mu^{[k_2]} (Z_j) \bigr\|^2 \Bigr]
\end{aligned} \label{eq:T_2} 
\end{align}
%
%
where we recall that for $z \equiv (t,x) \in [0,T) \times \mathbb{R}^D$, $\mu^{[i]} (z) = \mathbb{E} \bigl[~\for{X}{0} | \for{X}{T-t} = x, \, \for{X}{0} \sim P_{\mathrm{data}}^{[i]} \bigr], ~i \in [K]$. 

\subsection{Step 3. Bounding the Term $\mathcal{T}_1$} \label{sec:pf_step_3} 
The eigen-decomposition $\Sigma_k \equiv U \Lambda_k U^\top$ together with Assumption \ref{assump:cov} gives that: for $j = 0, 1, \ldots, N-1$,  
\begin{align}
& \sigma_{T-t_j}^2 \mathrm{tr} \left( \Sigma_k \Bigl( m_{T-t_j}^2 \Sigma_k + \sigma_{T-t_j}^2 I_D \Bigr)^{-1} \right) 
= \sigma_{T-t_j}^2 \mathrm{tr} \left( \Lambda_k \Bigl( m_{T-t_j}^2 \Lambda_k + \sigma_{T-t_j}^2 I_D   \Bigr)^{-1} \right) \nonumber \\ 
& \qquad \qquad = \sigma_{T-t_j}^2 \sum_{\ell = 1}^D \tfrac{\lambda_\ell (\Sigma_{k})}{m_{T-t_j}^2 \lambda_\ell (\Sigma_{k}) + \sigma_{T-t_j}^2} 
\le \tfrac{\sigma_{T-t_j}^2}{m_{T-t_j}^2} M_k + \nu_k. \label{eq:trace_inv_bd}
\end{align} 
Thus, we get
\begin{align}
\mathcal{T}_1 
& \le \tau \snr{1}
\sum_{k = 1}^K \mathbb{E} \bigl[ W_D^{[k]} (Z_0) \bigr] \cdot \Bigl( 
\tfrac{\sigma_{T}^2}{m_{T}^2} M_k + \nu_k \Bigr) 
+ \tau \sum_{j = 1}^{N-1} \Delta_j 
\sum_{k = 1}^K \mathbb{E} \bigl[ W_D^{[k]} (Z_j) \bigr] \cdot \Bigl( 
\tfrac{\sigma_{T-t_j}^2}{m_{T-t_j}^2} M_k + \nu_k \Bigr) \nonumber \\
& \le \tau \Bigl( \tfrac{\snr{1}}{\snr{0}} + \snr{1} R \Bigr) M^\dagger + \tau \sum_{j = 1}^{N-1} \tfrac{\Delta_j}{\snr{j}} \bigl(1 + \snr{j} R \bigr) M^\dagger \nonumber \\ 
& = \tau \Bigl( \tfrac{\snr{1}}{\snr{0}} + \sum_{j = 1}^{N-1} \tfrac{\Delta_j}{\snr{j}} \Bigr) M^\dagger 
+ \tau \snr{N} R,  \nonumber 
\end{align} 
where $M^\dagger = \max_{k \in [K]} M_k$, $R \in [0, 1)$ is a universal constant defined in Assumption \ref{assump:cov} 
and we used $\sum_{k = 1}^K W_D^{[k]} (t, x) =1$ for all $(t, x) \in [0, T) \times \mathbb{R}^D$. 
We then get under $\varepsilon \gtrsim R/T$ that  
\begin{align} \label{eq:T_1_bd}  
\mathcal{T}_1 \lesssim \tau \Bigl( 1 + \log \tfrac{\snr{N}}{\snr{1}} \Bigr) M^\dagger + \tau \tfrac{R}{\varepsilon} 
\lesssim \tau \Bigl( \log \tfrac{1}{\varepsilon} + T \Bigr) M^\dagger, 
\end{align}
since it follows that for $j \in [N-1]$, 
\begin{align} \label{eq:rel_err_snr}
\tfrac{\Delta_j}{\snr{j}} 
= \tfrac{\snr{j+1}}{\snr{j}} \int_{\snr{j}}^{\snr{j+1}} \tfrac{1}{\snr{j+1}} dy \le \tfrac{\snr{j+1}}{\snr{j}} \int_{\snr{j}}^{\snr{j+1}} \tfrac{1}{y} dy 
\lesssim \log \tfrac{\snr{j+1}}{\snr{j}}
\end{align}
with $\snr{k+1}/\snr{k} \le 4, \, k = 0, 1, \ldots, N-1$ from Lemma \ref{lemma:noise_ratio_bd} and 
\begin{align*}
& \log \tfrac{1}{\snr{1}} \le - \log m_{T-t_1}^2 = 2 (T-t_1) \lesssim T,  \\ 
& \log {\snr{{N}}}  = \log \Bigl( \tfrac{1}{\exp (2 \varepsilon) - 1} \Bigr) \le \log \tfrac{1}{\varepsilon}, 
\quad \bigl( \because \exp(x) - 1 \ge x, \forall x > 0 \bigr).  
\end{align*}  
\subsection{Step 4. Bounding the Term $\mathcal{T}_2$} \label{sec:pf_step_4}  
We hereby focus on bounding the term involving the inter-class difference between clusters, i.e., $\bigl\| \mu^{[k_1]} (Z_j) - \mu^{[k_2]} (Z_j) \bigr\|^2$ that can depend on $\| \mu_i - \mu_j \|^2 = \mathcal{O} (D)$ under Assumption \ref{assump:mean}. To avoid the polynomial dependence on $D$ in the bound, we make use of the phase transition from mixing to single-cluster commitment, as argued in Section \ref{sec:phase_transition}. We thus separate the time interval $[0, T-\varepsilon]$ into two parts and get $\mathcal{T}_2 = \mathcal{T}_2^{I} + \mathcal{T}_2^{II}$ with 
\begin{align}
\mathcal{T}_2^{I} 
& = \tfrac{\tau}{2} \sum_{j = 0}^{N_1-1} \bigl( \snr{1} \mathbf{1}_{j = 0} +  \Delta_j \mathbf{1}_{j \ge 1} \bigr) 
\! \! 
\sum_{\substack{k_1, k_2 = 1}}^K  
\! \! \mathbb{E} \Bigl[ W_D^{[k_1]} (Z_j) W_D^{[k_2]} (Z_j) \bigl\| \mu^{[k_1]} (Z_j)
- \mu^{[k_2]} (Z_j) \bigr\|^2 \Bigr]; \label{eq:T_2_I} \\[0.2cm] 
\mathcal{T}_2^{II} 
& = \tfrac{\tau}{2} \sum_{j = N_1}^{N-1} \Delta_j \sum_{\substack{k_1, k_2 = 1}}^K  
\mathbb{E} \Bigl[ W_D^{[k_1]} (Z_j) W_D^{[k_2]} (Z_j) \bigl\| \mu^{[k_1]} (Z_j)
- \mu^{[k_2]} (Z_j) \bigr\|^2 \Bigr], \label{eq:T_2_II} 
\end{align}
where in the above we have set: 
\begin{align} \label{eq:N_1}
N_1 \equiv \inf \Bigl\{ j \in [N] \, | \, \snr{j} > L \tfrac{\log (KD)}{D} \Bigr\} 
\end{align}
for an arbitrary fixed constant $L > 0$ satisfying $L \log (KD) / D \ge \underline{S}$, with $\underline{S}$ specified in the statement of Theorem \ref{thm:NA_transition_est}. We note that $N_1 \ge 1$ under $T \ge \tfrac{1}{2} \log D$ due to Corollary \ref{cor:T_scale} and $\snr{j} \lesssim  \tfrac{\log (KD)}{D}$ for $j \in \{1, \ldots, N_1 - 1\}$. In what follows, we write $Z_j^{[i]} = (t_j, \for{X}{T-t_j}^{[i]}), \ j \in \{0, 1, \ldots, N-1 \}, \, i \in [K]$ where $\for{X}{T-t_j}^{[i]}$ follows the distribution $\for{p}{T-t_j}^{[i]}$.

\subsubsection{Bounding the term $\mathcal{T}_2^I$}
We begin with the term $\mathcal{T}_2^I$, that is, the error during the mixing phase. We then get that 
\begin{align*}
\mathcal{T}_2^I = \tfrac{\tau}{2} \sum_{j = 0}^{N_1-1} \bigl( \snr{1} \mathbf{1}_{j = 0} + \Delta_j \mathbf{1}_{j \ge 0} \bigr) \sum_{i \in [K]} \sum_{\substack{k_1, k_2 = 1}} \pi_i \cdot Q_{k_1, k_2}^{[i]} (j)
\end{align*}
with 
\begin{align*}
Q_{k_1, k_2}^{[i]} (j) = \mathbb{E} \Bigl[ W_D^{[k_1]} (Z_j^{[i]}) W_D^{[k_2]} (Z_j^{[i]}) \bigl\| \mu^{[k_1]} (Z_j^{[i]})
- \mu^{[k_2]} (Z_j^{[i]}) \bigr\|^2 \Bigr], 
\end{align*}
where the expectation operator acts on the second component of the variable $Z_{j}^{[i]}$ under the law of $\for{X}{T-t_j}^{[i]}$. We have that 
\begin{align} 
& \sum_{k_1, k_2 \in [K]} Q_{k_1, k_2}^{[i]} (j) 
\le \mathbb{E} \bigl[ \max_{k_1, k_2 \in [K]} \bigl\| \mu^{[k_1]} (Z_j^{[i]})
- \mu^{[k_2]} (Z_j^{[i]}) \bigr\|^2 \bigr] \qquad 
\Bigl( \because \sum_{k \in [K]} W_D^{[k]} (Z_j^{[i]}) = 1 \Bigr) \nonumber \\[0.2cm] 
& \quad = \tfrac{1}{\xi_j} \log \exp \Bigl( \mathbb{E} \bigl[ \xi_j \cdot \max_{k_1, k_2 \in [K]} \bigl\| \mu^{[k_1]} (Z_j^{[i]})
- \mu^{[k_2]} (Z_j^{[i]}) \bigr\|^2 \bigr] \Bigr) \nonumber \\[0.2cm] 
& \quad \le \tfrac{1}{\xi_j} \log \Bigl(  \mathbb{E} \Bigl[ \exp \Bigl( \xi_j \max_{k_1, k_2 \in [K]}  \bigl\| \mu^{[k_1]} (Z_j^{[i]})
- \mu^{[k_2]} (Z_j^{[i]}) \bigr\|^2 \Bigr) \Bigr]  \Bigr) \qquad \bigl(\because \mathrm{Jensen's \, inequality } \bigr) \nonumber \\[0.2cm] 
& \quad = \tfrac{1}{\xi_j} \log \Bigl(  \mathbb{E} \Bigl[ \max_{k_1, k_2 \in [K]}  \exp \Bigl( \xi_j \bigl\| \mu^{[k_1]} (Z_j^{[i]})  - \mu^{[k_2]} (Z_j^{[i]}) \bigr\|^2 \Bigr) \Bigr]  \Bigr) \label{eq:log_sum_exp}
\end{align} 
for any $\xi_j > 0$  and for each fixed $i \in [K]$ and $j \in \{0, 1, \ldots, N_1 - 1\}$. To further bound the above term, we exploit the following result whose proof is postponed to Appendix \ref{sec:pf_max_chi_squared}: 
\begin{lemma} \label{lemma:max_chi_squared}
Under Assumptions \ref{assump:cov}, \ref{assump:mean}, we have the following: for any $i \in [K]$,  $j \in \{0, 1, \ldots, N-1\}$ and any $\rho \in \bigl( 0, m_{T-t_j}^2/ (16 \sigma_{T-t_j}^4  \max_{i, k \in [K]} \| H_{k, i} (j) \|_2) \bigr]$, it holds 
\begin{align*}
\log \Bigl(  \mathbb{E} \Bigl[ \max_{k_1, k_2 \in [K]}  \exp \Bigl( \rho \bigl\| \mu^{[k_1]} (Z_j^{[i]})
- \mu^{[k_2]} (Z_j^{[i]}) \bigr\|^2 \Bigr) \Bigr]  \Bigr) \lesssim \log K + \rho \bigl(D + \snr{j} M^\dagger D^{2 \alpha} \bigr) 
\end{align*}
where we have set  
\begin{align} \label{eq:H}
H_{k, i} (j) = m_{T-t_j}^4 \bigl\{ (S_{T-t_j}^k)^{-1} (\Sigma_k - \Sigma_i) (S_{T-t_j}^i)^{-1} (\Sigma_k - \Sigma_i) (S_{T-t_j}^k)^{-1} \bigr\}. 
\end{align}
\end{lemma}
%
Setting $\xi_j = m_{T-t_j}^2 / (16 \sigma_{T-t_j}^4 \max_{i, k \in [K]} \| H_{k, i} (j) \|_2)$ in (\ref{eq:log_sum_exp}), we apply Lemma \ref{lemma:max_chi_squared} to get
\begin{align}
\mathcal{T}_2^I 
& \lesssim \tau \tfrac{\snr{1}}{\xi_0} \bigl( \log K +  \xi_0 (D + \snr{j} M^\dagger D^{2 \alpha}) \bigr) 
+ \tau \sum_{j = 1}^{N_1 - 1}  
\tfrac{\Delta_j}{\xi_j} \bigl( \log K +  \xi_j (D + \snr{j} M^\dagger D^{2 \alpha}) \bigr) \nonumber  \\ 
& \lesssim \tau \Bigl(\snr{1} + \sum_{j = 1}^{N_1 - 1} \Delta_{j} \Bigr) \Bigl(  (\log K) \snr{N_1 - 1} D^{2 \alpha} + D + \snr{N_1 - 1} M^\dagger D^{2 \alpha} \Bigr)  \nonumber  \\ 
& =  \tau \snr{N_1} \Bigl(\snr{N_1 - 1} \bigl( \log K + M^\dagger \bigr) D^{2 \alpha} + D \Bigr)  \nonumber  \\ 
& \lesssim \tau \snr{N_1 - 1} \Bigl(\snr{N_1 - 1} \bigl( \log K + M^\dagger \bigr) D^{2 \alpha} + D \Bigr)  \quad 
\bigl( \because \mathrm{Lemma \, \ref{lemma:noise_ratio_bd}} \bigr)
\nonumber \\ 
& \lesssim \tau \Bigl[
\Bigl(\tfrac{\log (KD)}{D^{1 - \alpha}} \Bigr)^2  \, 
\bigl( \log K  + M^\dagger \bigr) + \log (KD) \Bigr] \qquad \bigl( \because \snr{N_1 - 1} \lesssim \tfrac{\log (KD)}{D} \bigr)
\label{eq:T_2_mixing_bd}
\end{align} 
where in the second line we used 
\begin{align} \label{eq:bd_inv_xi}
    \tfrac{1}{\xi_j} \lesssim \tfrac{\sigma_{T-t_j}^4}{m_{T-t_j}^2} \max_{i, k \in [K]} \| H_{k, i} \|_2 \lesssim \snr{j} D^{2 \alpha} \lesssim \snr{N_1 - 1} D^{2 \alpha}. 
\end{align} 
\subsubsection{Bounding the term $\mathcal{T}_2^{II}$} 
We then consider the error during the cluster commitment phase, that is, $\mathcal{T}_2^{II}$ defined in (\ref{eq:T_2_II}). 
In this phase, we can exploit the high-probability concentration of the posterior weight stated in Theorem \ref{thm:NA_transition_est}. Introducing an event set $E_j^{[i]} (\eta) \subseteq \Omega, \, \eta \in (0,1)$ as  
\begin{align*} 
E_j^{[i]} (\eta) = \Bigl\{ \omega \in \Omega \ | \ \bigcap_{k \in [K] \setminus \{i\}} \bigl\{ W^{[k]}_D (Z_j^{[i]}) \le \exp 
\bigl( -\eta \KL{\for{p}{T-t_j}^{[i]}}{\for{p}{T-t_j}^{[k]}} \bigr) \bigr\} \Bigr\},
\end{align*}
we express the term $\mathcal{T}_2^{II}$ as $\mathcal{T}_2^{II} = \mathcal{T}_2^{II, E} + \mathcal{T}_2^{II, E^c}$ with 
%
\begin{align}
\mathcal{T}_2^{II, E} 
& = \tfrac{\tau}{2} \sum_{j = N_1}^{N-1} \Delta_j \sum_{i \in [K]} \sum_{\substack{k_1, k_2 \in [K]}} 
\pi_i \cdot 
Q_j^{[i], k_1, k_2} \bigl( E_j^{[i]}(\eta) \bigr); \label{eq:T_2_high_prob} \\ 
\mathcal{T}_2^{II, E^c} 
& = \tfrac{\tau}{2} \sum_{j = N_1}^{N-1} \Delta_j \sum_{i \in [K]}  \sum_{\substack{k_1, k_2 \in [K]}} \pi_i \cdot Q_j^{[i], k_1, k_2} \bigl( E_j^{[i]}(\eta)^c \bigr),  \label{eq:T_2_small_prob}
\end{align}
where we have set: for an event $A \subseteq \Omega$, 
\begin{align*}
Q_j^{[i], k_1, k_2} (A) = \mathbb{E} \Bigl[ W_D^{[k_1]} (Z_j^{[i]}) W_D^{[k_2]} (Z_j^{[i]}) \, \bigl\| \mu^{[k_1]} (Z_j^{[i]}) - \mu^{[k_2]} (Z_j^{[i]}) \bigr\|^2 \, \mathbf{1}_{A} \Bigr]. 
\end{align*}
%
%
For notational simplicity, we write: 
\begin{align*}
D_k^{[i]} (j) = \mathbb{E}\Bigl[ \bigl\| \mu^{[k]} (Z_j^{[i]}) - \mu^{[i]} (Z_j^{[i]}) \bigr\|^2 \Bigr], \qquad k \in \{k_1, k_2\}.  
\end{align*} 
Then, for the term $\mathcal{T}_2^{II, E} $, we have the following bound on the event set $E_j^{[i]} (\eta)$: 
\begin{align*}
Q_j^{[i], k_1, k_2} (E_j^{[i]}(\eta)) 
& \lesssim  
\exp \Bigl( - \eta \bigl\{ \KL{\for{p}{T-t_j}^{[i]}}{\for{p}{T-t_j}^{[k_1]}} + \KL{\for{p}{T-t_j}^{[i]}}{\for{p}{T-t_j}^{[k_2]}} \bigr\} \Bigr) 
\cdot \bigl\{ D_{k_1}^{[i]} (j) +  D_{k_2}^{[i]}(j) \bigr\} \\[0.1cm]
&  \lesssim  \bigl( \| \Sigma_i \|_2  + \snr{j}^{-1} \bigr) \exp \Bigl( - \eta \bigl\{ \KL{\for{p}{T-t_j}^{[i]}}{\for{p}{T-t_j}^{[k_1]}} + \KL{\for{p}{T-t_j}^{[i]}}{\for{p}{T-t_j}^{[k_2]}} \bigr\} \Bigr) \\
& \qquad \qquad \qquad \qquad \qquad \qquad \quad \times  \Bigl\{ \KL{\for{p}{T-t_j}^{[i]}}{\for{p}{T-t_j}^{[k_1]}} + \KL{\for{p}{T-t_j}^{[i]}}{\for{p}{T-t_j}^{[k_2]}} \Bigr\} \\[0.1cm]  
& \lesssim \bigl( \| \Sigma_i \|_2  + \snr{j}^{-1} \bigr) \exp \Bigl( - \widetilde{\eta} \cdot \bigl\{ \KL{\for{p}{T-t_j}^{[i]}}{\for{p}{T-t_j}^{[k_1]}} + \KL{\for{p}{T-t_j}^{[i]}}{\for{p}{T-t_j}^{[k_2]}} \bigr\} \Bigr) \\ 
& \le \bigl( \| \Sigma_i \|_2  + \snr{j}^{-1} \bigr) \exp \Bigl( -  \widetilde{\eta} \cdot \Bigl(  \tfrac{ \widetilde{C}^2 \snr{j} D^2}{ \snr{j} \max_{k \in [K]} \mathrm{tr} (\Sigma_k)+ D} \Bigr) \Bigr)  \quad \bigl(\, \because   \mathrm{Lemma} \, \ref{lemma:klm_lbd} \, \bigr)
\end{align*}
for some $\widetilde{\eta} \in (0, \eta)$ and $\widetilde{C} = C_3 / C_5$, where in the second line above we used the following result whose proof is postponed to Appendix \ref{sec:pf_second_moment_bd}:  
\begin{lemma} \label{lemma:second_moment_bd}
It holds that: 
\begin{align}
\begin{aligned} 
D_k^{[i]} (j) \le
4 \bigl( \| \Sigma_i \|_2  + \snr{j}^{-1} \bigr) \cdot   \KL{\, \for{p}{T-t_j}^{[i]}}{\, \for{p}{T-t_j}^{[k]}}.  
\end{aligned} \label{eq:second_moment_bd}
\end{align} 
\end{lemma}
To further bound the term $Q_j^{[i], k_1, k_2} (E_j^{[i]}(\eta))$ under the case $j \in \{N_1, \ldots, N - 1 \}$, we define an integer $N_2$ as
\begin{align} \label{eq:N_2}
    N_2  & 
    =
    \begin{cases} 
    \inf \bigl\{ j = N_1, N_1 + 1, \ldots, N | \, \snr{j} > {D}/{\max_{k \in [K]} \mathrm{tr} (\Sigma_k)} \bigr\}, & \mathrm{if~the~set~is~ not~empty}; \\[0.1cm]
    N, & \mathrm{otherwise}. 
    \end{cases} 
\end{align}
We then consider the following two cases for $j$: (a). $j = N_1, \ldots, N_2 - 1$ and (b). $j = N_2, \ldots, N - 1$, where case (b) is understood to be empty when $N_2 = N$. For case (a), we have $\snr{j} \| \Sigma_i \|_2 \le \snr{j} \max_{k \in [K]} \mathrm{tr} (\Sigma_k) \le D$ and thus  
\begin{align}
Q_j^{[i], k_1, k_2} (E_j^{[i]}(\eta)) 
& \lesssim  \snr{j}^{-1} (D + 1)  \exp \Bigl( - \tfrac{\widetilde{\eta}  \widetilde{C}^2}{2} \snr{j} D \Bigr) \nonumber \\[0.1cm] 
& \le \snr{j}^{-1} \bigl(D +1 \bigr) \cdot (KD)^{-cL}, \qquad \bigl( \because \snr{j} \ge L {\log (KD)}/{D} \bigr),  \label{eq:T_2_commit_large_event_I}
\end{align} 
where $c = {\widetilde{\eta}  \widetilde{C}^{\, 2}}/{2}$. On the other hand, under case (b), we have 
\begin{align}
& Q_j^{[i], k_1, k_2} (E_j^{[i]}(\eta))  \nonumber \\ 
& \lesssim \snr{j}^{-1} \bigl( \snr{j} D^\alpha + 1 \bigr) \exp \Bigl( 
- \tfrac{\widetilde{\eta} \widetilde{C}^2}{2} \cdot \tfrac{D^2}{\max_{k \in [K]} \mathrm{tr} (\Sigma_k)} \Bigr) \quad \quad \bigl( \, \because \, \mathrm{Assumption} \, \ref{assump:cov} \bigr) \nonumber \\ 
& \lesssim \snr{j}^{-1} \cdot 
\tfrac{D^\alpha}{\varepsilon} \cdot \exp \Bigl( 
-  \tfrac{\widetilde{\eta} \widetilde{C}^2}{2} \cdot \tfrac{D^2}{M^\dagger (C_2 D^\alpha + R)} \Bigr) 
\quad \bigl( \, \because \, \mathrm{Assumption} \, \ref{assump:cov} \, \& \, \snr{j} \le \snr{\varepsilon} \lesssim \tfrac{1}{\varepsilon} \bigr) \nonumber \\ 
& \le \snr{j}^{-1} \cdot \tfrac{D^\alpha}{\varepsilon} \cdot \exp \Bigl( 
- \tfrac{\widetilde{\eta} \widetilde{C}^2}{4 M^\dagger C_2} \cdot {D^{2- \alpha}} \Bigr) 
\lesssim \snr{j}^{-1} \cdot \tfrac{D^{2 \alpha - 2}}{\varepsilon} \cdot \exp \Bigl( 
- \widetilde{c} \cdot {D^{2- \alpha}} \Bigr) \label{eq:T_2_commit_large_event_II} 
\end{align}
for some $\widetilde{c} \in (0, {\widetilde{\eta} \widetilde{C}^2}/{4 M^\dagger C_2}).$ Substituting the bounds (\ref{eq:T_2_commit_large_event_I}) and (\ref{eq:T_2_commit_large_event_II}) into (\ref{eq:T_2_high_prob}), we obtain
\begin{align}
\mathcal{T}^{II, E}_2 
& \lesssim \tau \Biggl[ K^2 D (KD)^{-cL} \cdot \sum_{j = N_1}^{N_2 - 1} \tfrac{\Delta_j}{\snr{j}} 
+ \tfrac{K^2}{\varepsilon} \exp (- \widetilde{c} \, D^{2 - \alpha}) \cdot \sum_{j = N_2}^{N - 1} \tfrac{\Delta_j}{\snr{j}} \Biggr] \nonumber \\ 
& \lesssim \tau \Bigl[ \bigl( \log \tfrac{1}{\varepsilon} + \log D \bigr) \cdot \Bigl\{ K^2 D (KD)^{-cL} + \tfrac{K^2}{\varepsilon} \exp (- \widetilde{c} \, D^{2 - \alpha}) \Bigr\} \Bigr],  \label{eq:T_2_commit_high_prob_bd}
\end{align}
where we used 
\begin{align} \label{eq:rel_diff_snr_bd_II}
    \max \Bigl\{ \sum_{j = N_1}^{N_2 - 1} \tfrac{\Delta_j}{\snr{j}},  \sum_{j = N_2}^{N - 1} \tfrac{\Delta_j}{\snr{j}} \Bigr\} \le \sum_{j = N_1}^{N-1} \tfrac{\Delta_j}{\snr{j}} \le \log \tfrac{1}{\varepsilon} + \log D. 
\end{align}
We now consider the term $\mathcal{T}_2^{II, E^c}$ defined in (\ref{eq:T_2_small_prob}). For ease of notation, we write $B = (E_j^{[i]} (\eta))^c$ and $\mathbb{P}_1 = \for{p}{T-t_j}^{[i]}$, and  introduce a probability measure $\mathbb{P}_2$ defined as $\mathrm{d} {\mathbb{P}_2} = \tfrac{\mathbf{1}_{B}}{\mathbb{P}_1 (B)} \mathrm{d} \mathbb{P}_1$. We then get
\begin{align*}
& \sum_{k_1, k_2 = 1}^K Q_{j}^{[i], k_1, k_2} (B)  
\le \mathbb{E}_{\mathbb{P}_1} \Bigl[ 
\max_{k_1, k_2 \in [K]} \bigl\| \mu^{[k_1]} (Z_j^{[i]})
- \mu^{[k_2]} (Z_j^{[i]}) \bigr\|^2 \cdot \mathbf{1}_{B} \Bigr] \\[0.2cm]
& \quad = \tfrac{\mathbb{P}_1 \bigl( B \bigr)}{\xi_j} \cdot \mathbb{E}_{\mathbb{P}_2} \Bigl[ \xi_j \max_{k_1, k_2 \in [K]} \bigl\| \mu^{[k_1]} (Z_j^{[i]}) - \mu^{[k_2]} (Z_j^{[i]}) \bigr\|^2 \Bigr]  \\[0.2cm]
& \quad \le \tfrac{\mathbb{P}_1 \bigl( B \bigr)}{\xi_j} \cdot 
\Biggl\{ \log \mathbb{E}_{\mathbb{P}_1} \Bigl[ \exp \Bigl(\xi_j \max_{k_1, k_2 \in [K]} \bigl\| \mu^{[k_1]} (Z_j^{[i]})
- \mu^{[k_2]} (Z_j^{[i]}) \bigr\|^2 \Bigr) \Bigr] 
+ \KL{\mathbb{P}_2}{\mathbb{P}_1} \Biggr\}, 
\end{align*}
for $\xi_j > 0$ such that $\mathbb{E}_{\mathbb{P}_1} \Bigl[ \exp \Bigl(\xi_j \max_{k_1, k_2 \in [K]} \bigl\| \mu^{[k_1]} (Z_j^{[i]}) - \mu^{[k_2]} (Z_j^{[i]}) \bigr\|^2 \Bigr) \Bigr] < \infty$, where the second bound is obtained via Donsker-Varadhan's variational formula \citep{Dupuis:97}. We immediately have  $\KL{\mathbb{P}_2}{\mathbb{P}_1} = - \log \mathbb{P}_1 (B)$. 
Also, Theorem \ref{thm:NA_transition_est} gives that for any $q \ge 1, \eta \in (0, 1)$, there exists $D^\dagger \in \mathbb{N}$ such that for any  $D \ge D^\dagger$, we have $\mathbb{P}_1 (B) \le D^{-q}$. Thus, by setting $\xi_j = m_{T-t_j}^2 / (16 \sigma_{T-t_j}^4 \max_{i, k \in [K]} \| H_{k, i} (j) \|_2)$, we obtain: 
\begin{align}
\sum_{k_1, k_2 = 1}^K Q_{j}^{[i], k_1, k_2} (B) 
& \lesssim \mathbb{P}_1 (B) \cdot \Bigl\{ \tfrac{1}{\xi_j} \log K  + D + \snr{j} M^\dagger D^{2 \alpha} +  \tfrac{1}{\xi_j} \log \tfrac{1}{\mathbb{P}_1 (B)} \Bigr\}
\quad \bigl( \because \mathrm{\, Lemma\,} \ref{lemma:max_chi_squared} \bigr) \nonumber \\
& \lesssim D^{2 \alpha -q} \cdot \snr{j} \cdot \bigl( \log K + M^\dagger  \bigr) + D^{1-q} + \snr{j} D^{2\alpha}~ \mathbb{P}_1 (B) \log \tfrac{1}{\mathbb{P}_1 (B)} \quad \bigl(\, \because (\ref{eq:bd_inv_xi}) \bigr) \nonumber \\[0.1cm]   
& \lesssim D^{2 \alpha -q} \cdot \snr{j} \cdot \bigl( \log K + M^\dagger \bigr)  + D^{1-q} + D^{2 \alpha -q} \cdot \snr{j} \cdot \log D \nonumber \\[0.1cm] 
& \lesssim \snr{j}^{-1} \bigl( D^{2 \alpha - q} \cdot \varepsilon^{-2} \cdot (\log K + M^\dagger + \log D) +  D^{1-q} \cdot \varepsilon^{-1} \bigr), \label{eq:T_2_commit_small_event}
\end{align} 
where we used for $x \log (1/x) \le r \log (1/r)$ $x \in [0, r], r \le e^{-1}$, and $\snr{j} \lesssim \varepsilon^{-1}$ in the third and fourth inequality, respectively. Thus, the bound \eqref{eq:T_2_commit_small_event} together with (\ref{eq:rel_diff_snr_bd_II}) leads to 
\begin{align}
\mathcal{T}_2^{II, E^c} 
& \lesssim \tau \Bigl[ \bigl( D^{2 \alpha - q} \cdot \varepsilon^{-2} \cdot (\log K + M^\dagger + \log D) +  D^{1-q} \cdot \varepsilon^{-1} \bigr) \sum_{j = N_1}^N \tfrac{\Delta_j}{\snr{j}}
 \Bigr] \nonumber \\ 
 & \lesssim \tau \Bigl[ \bigl( D^{2 \alpha - q} \cdot \varepsilon^{-2} \cdot (\log K + M^\dagger + \log D) +  D^{1-q} \cdot \varepsilon^{-1} \bigr) \bigl( \log \tfrac{1}{\varepsilon} + \log D \bigr) 
 \Bigr].  \label{eq:T_2_commit_small_prob_bd}
\end{align}
\subsection{Step 5: Final Bound}
\label{sec:pf_step_5} 
%
Combining the upper bounds \eqref{eq:T_1_bd}, 
\eqref{eq:T_2_mixing_bd}, \eqref{eq:T_2_commit_high_prob_bd} and \eqref{eq:T_2_commit_small_prob_bd} into the term $R_2^{II}$ defined in (\ref{eq:R_2_II}), the time-discretisation error is bounded as follows: 
\begin{align}
\begin{aligned} \label{eq:eq:time_disc_bd}
    R_2^{II} 
    & \lesssim \mathcal{T}_1 + \mathcal{T}_2^I + \mathcal{T}_2^{II, E} + \mathcal{T}_2^{II, E^c} \\  
    & \lesssim \tau \Bigl[ M^\dagger \Bigl( \log \tfrac{1}{\varepsilon} + T \Bigr) + \Bigl(\tfrac{\log (KD)}{D^{1 - \alpha}} \Bigr)^2  \,  \bigl( \log K  + M^\dagger \bigr) + \log (KD) \\
    & \quad \quad + \bigl( \log \tfrac{1}{\varepsilon} + \log D \bigr) \cdot \Bigl\{ K^2 D (KD)^{-cL} + \tfrac{K^2}{\varepsilon} \exp (- \widetilde{c} \, D^{2 - \alpha}) \Bigr\} \\ 
    & \quad \quad + \bigl( D^{2 \alpha - q} \cdot \varepsilon^{-2} \cdot (\log K + M^\dagger + \log D) +  D^{1-q} \cdot \varepsilon^{-1} \bigr) \bigl( \log \tfrac{1}{\varepsilon} + \log D \bigr) \Bigr]. 
\end{aligned}
\end{align} 
For the number mixture of $K$ that grows polynomially in  $D$, by choosing sufficiently large $q \ge 2$ and $L > 0$, we have under $T \ge \tfrac{1}{2} \log D$ that for any sufficiently large $D$, 
\begin{align} \label{eq:time_disc_bd_simplified}
R_2^{II} \lesssim \tau  M^\dagger \Bigl( \log \tfrac{1}{\varepsilon} + T \Bigr). 
\end{align}
We thus conclude Theorem \ref{thm:main} from \eqref{eq:prior_gap}, \eqref{eq:score_est_gap} and \eqref{eq:time_disc_bd_simplified}. 

\section{Conclusion and Discussion} \label{sec:conclusion}
This work has analytically studied the geometric adaptivity of denoising diffusion models when the target distributions to sample from are characterised as a union of well-separated classes with their own low-dimensional structure, motivated by the \emph{union of manifolds hypothesis} \citep{brown:22}. By focusing on $K$-Gaussian mixture models, our analysis has focused on the dynamical behaviour of the score function and, in particular, the behaviour of the posterior weights for sampling from each cluster component. 
Under the conditions, particularly the  $\Theta(D)$ inter-class separation in cluster means, we have shown that the concentration of the posterior weights occurs with high probability after the SNR reaches a threshold ${\Theta} \bigl( \log (KD) / D \bigr)$. We then use this result to obtain a KL divergence error bound, where the time-discretisation error is shown to scale linearly with the intrinsic dimensionality rather than the ambient dimensionality $D$. The main difference between our analysis and the literature (e.g., \cite{ben:24, conf:25}) is that our analysis accounts for the dynamical regimes of the denoising process. Specifically, during the mixing phase (before the posterior commitment), the small SNR suppresses the effect of the inter-class separation of means and covariances in terms of $D$; whereas after the cluster commitment, when the SNR is no longer sufficiently small, the posterior weight decay neutralises the large cluster separations. We also note that the literature on diffusion models targeting GMMs \citep{li:26} establishes a dimension-free bound under shared isotropic covariances, while our work treats GMMs with heterogeneous anisotropic covariances and derives a desirable error bound by employing the phase transition. Our analysis exploits $\Theta(D)$ separation in the mean vectors to induce the phase transition, with the key SNR scaling with $D$, but it could be extended to cases where covariance separation plays a central role in the transition.   

The numerical experiments in Section \ref{sec:numerical_experiments} demonstrate the SNR scaling $\Theta(\log D/D)$ in the phase transition for real high-dimensional datasets beyond our geometric framework (GMMs). Given this, our phase transition-driven error analysis can be applied to a wider class of mixture models where we believe similar intrinsic-dimension adaptivity would hold, provided we can still prove concentration results on the posterior weights and the score function for each mixture component is sufficiently well-behaved. The former is likely to hold if the mixture components are sub-Gaussian, while the latter would apply under assumptions such as strong log-concavity of the cluster densities.

Another interesting extension would be to hierarchical clusters. For a mixture model with `super-components', each of which is a Gaussian mixture model, with $\mathcal{O}(D)$ separation between the super-component means, and smaller separation, say $\mathcal{O}(1)$, between the means of cluster components within each super-component. For this setting, we would still get concentration of the posterior weights to a super-component with high probability. Then, the KL divergence error is controlled by the intrinsic dimensionality in a similar manner, so that the relatively small separation within each super-component does not degrade the dimensional scaling in the error bound.

\subsection*{Acknowledgement} 
{We thank Prof Kenji Fukumizu and Prof Paul Jenkins for detailed comments and feedback on this work. The authors acknowledge support from EPSRC grant EP/Y028783/1.} 

\vskip 0.2in

\bibliographystyle{plainnat}
\bibliography{main}

\section*{Appendix}

\appendix 

\numberwithin{theorem}{section} 

\section{Property of Time-Discretisation Schedule} \label{sec:pf_N_scale} 
To show Lemma \ref{lemma:complexity_N}, we first prove the following result, which is also key in the proof of Theorem \ref{thm:main}: 
\begin{lemma} \label{lemma:noise_ratio_bd}
For the time-discretisation schedule (\ref{eq:t_schedule}), it holds that: 
\begin{align*}
\tfrac{\sigma_{T-t_j}^2}{\sigma_{T-t_{j+1}}^2} \le 2, \quad
\tfrac{\snr{T-t_{j+1}}}{\snr{T-t_j}} \le 4,  \qquad  j = 0, \ldots, N-1. 
\end{align*} 
\end{lemma}
\begin{proof}
Noticing that $h_{i+1} = \tau \sigma_{T-t_i}^2, \, i = 0, 1, \ldots, N-2$ and $h_N \le \tau \sigma_{T-t_{N-1}}^2$, we have that for all $j = 0, 1, \ldots, N-1$, 
\begin{align*}
\sigma_{T-t_j}^2 - \sigma_{T-t_{j+1}}^2 = \int_{T-t_{j+1}}^{T-t_j} \tfrac{d}{ds} \sigma_{s}^2 ds 
\le 2 h_{j+1} \le 2 \tau \sigma_{T-t_j}^2, 
\end{align*} 
which leads to $\sigma_{T-t_j}^2 / \sigma_{T-t_{j+1}}^2 \le 1 / (1 - 2 \tau) \le 2$ for $\tau \in (0, 1/4]$. For the latter inequality, using $m_{T-t_{j+1}}^2/m_{T-t_j}^2 \le \exp \bigl( 2 h_{j+1} \bigr) \le \exp (2 \tau)$, we get
\begin{align*}
\tfrac{\snr{T-t_{j+1}}}{\snr{T-t_j}} 
= \tfrac{m_{T-t_{j+1}}^2}{m_{T-t_{j}}^2} \cdot \tfrac{\sigma_{T-t_{j}}^2}{\sigma_{T-t_{j+1}}^2} \le 2 \exp (2 \tau) \le 4, 
\end{align*} 
which concludes the proof. 
\end{proof}

\noindent 
(\textit{Proof of Lemma \ref{lemma:complexity_N}}) 
We consider the term $T_j \equiv \int_{t_j}^{t_{j+1}} \tfrac{ds}{\sigma^2_{T-s}}, \, j = 0, 1, \ldots, N-1$. Because $s \mapsto 1/ \sigma_{T-s}^2$ is an increasing function, (\ref{eq:t_schedule}) and Lemma \ref{lemma:noise_ratio_bd} give that 
$ T_j \le \tau {\sigma_{T-t_j}^2}/{\sigma_{T-t_{j+1}}^2} \le 2 \tau, \, j = 0, 1, \ldots, N-1$
where for $j = N-1$ we considered $h_{N} \le \tau \sigma_{T-t_{N-1}}^2$. On the other hand, we have 
\begin{align*}
    T_j \ge  
    \begin{cases}
    \tau & j = 0, 1, \ldots, N-2; \\ 
    0 & j = N-1. 
    \end{cases}
\end{align*}
Taking summation, we get $\tau (N-1) \le \sum_{j = 0}^{N-1} T_j \le 2 N \cdot \tau$ and 
\begin{align*}
\sum_{j = 0}^{N-1} T_j = \int_0^{T-\varepsilon} \tfrac{ds}{\sigma_{T-s}^2} = T-\varepsilon + \tfrac{1}{2} \log \frac{1 - \exp (- 2T)}{1 - \exp (- 2 \varepsilon)} = \Theta \bigl( T - 1 + \log \tfrac{1}{\varepsilon} \bigr),  
\end{align*}
which concludes Lemma \ref{lemma:complexity_N}. 

\section{Proofs for Posterior Classification Analysis}

\subsection{Proof of Lemma \ref{lemma:weight_bd_j}} \label{sec:pf_weight_bd}

The definition of posterior weight yields  
\begin{align*}
\log W_D^{[i]} (t, y) 
= \log \tfrac{\pi_i \cdot \for{p}{T-t}^{[i]} (y)}{\sum_{\ell \in [K]} \pi_\ell \cdot \for{p}{T-t}^{[\ell]} (y)}
\le \min \bigl\{ \log \tfrac{\pi_i}{\pi_j} + 
Z^{(i, j)}_t (y), 0 \bigr\}, \quad y \in \mathbb{R}^D, \, j \in [K],  
\end{align*} 
with $Z^{(i, j)}_t (y) = \log \tfrac{ \mathcal{N} (y; \, m_{T-t} \mu_{i}, S_{T-t}^{i})}{ \mathcal{N}(y; \, m_{T-t} \mu_{j}, S_{T-t}^{j})}$. 
The conditional random variable can be expressed as $\widetilde{Y}^{[j]}_{T-t} \sim \mathcal{N} (m_{T-t} \mu_j, S_{T-t}^j)$ due to $\mathrm{Law} (\for{X}{t}) = \mathrm{Law} (\back{X}{T-t}) $, where we recall $S_{T-t}^j$ in (\ref{eq:forward_noise_cov}). We then decompose $Z^{(i, j)}_t (\widetilde{Y}^{[j]}_{T-t}) = \Psi_1 + \Psi_2 + \Psi_3$ where 
\begin{align*} 
\Psi_1 & \equiv - m_{T-t} (\widetilde{Y}^{[j]}_{T-t} - m_{T-t} \mu_j)^\top (S_{T-t}^i)^{-1} (\mu_j - \mu_i); \\ 
\Psi_2 & \equiv \tfrac{1}{2} (\widetilde{Y}^{[j]}_{T-t} - m_{T-t} \mu_j)^\top \bigl( (S_{T-t}^j)^{-1} - (S_{T-t}^i)^{-1} \bigr) (\widetilde{Y}^{[j]}_{T-t} - m_{T-t} \mu_j)  ;\\ 
\Psi_3 & \equiv \tfrac{1}{2} \log \tfrac{\det S_{T-t}^j}{\det S_{T-t}^i} - \tfrac{m_{T-t}^2}{2} (\mu_j - \mu_i)^\top (S_{T-t}^i)^{-1} (\mu_j - \mu_i).  
\end{align*}
The first term follows a Gaussian distribution with mean $0$ and variance given as 
\begin{align*} 
\mathrm{Var} [\Psi_1] = m_{T-t}^2 (\mu_j - \mu_i)^\top 
(S_{T-t}^i)^{-1} (S_{T-t}^j) (S_{T-t}^i)^{-1} (\mu_j - \mu_i) \equiv \mathscr{T}_D^{[i, j]} (t).   
\end{align*}
Thus, the standard high probability argument for the Gaussian variable gives: for $\delta \in (0,1)$, 
\begin{align*}
\mathbb{P} \Bigl( | \Psi_1 | \ge \sqrt{2 \mathrm{Var} [\Psi_1] \log \tfrac{2}{\delta}} 
\Bigr) \le \tfrac{\delta}{2}.  
\end{align*}
The second term $\Psi_2$, quadratic w.r.t. the Gaussian $\widetilde{Y}_{T-t}^{[j]} - m_{T-t} \mu_j \sim \mathcal{N} (0, S_{T-t}^j)$ is expressed as $\Psi_2 = z^\top A z$ with 
\begin{align*}
A \equiv \tfrac{1}{2} \bigl( 
I_D - (S_{T-t}^j)^{1/2} (S_{T-t}^i)^{-1} (S_{T-t}^j)^{1/2}
\bigr), \qquad z \sim \mathcal{N} (0, I_D).   
\end{align*}
Then, application of the concentration inequality \cite[Proposition 1]{Hsu:12} to $\Psi _2$ yields: for $\delta \in (0,1)$,  
\begin{align*}
\mathbb{P} \Bigl( \Psi_2 \ge 
\mathrm{tr} (A) + \sqrt{4 \mathrm{tr} (A^2) \log \tfrac{2}{\delta}} + 2 \| A \|_{2} \log \tfrac{2}{\delta} \Bigr) \le \tfrac{\delta}{2}. 
\end{align*}
Thus with probability at least $1 - \delta$, we have: 
\begin{align*}
& Z_t^{(i, j)} (\widetilde{Y}_{T-t}^{[j]})  
 \le \Psi_3 + \mathrm{tr} (A) + \sqrt{2 \mathrm{Var}[\Psi_1] \log \tfrac{2}{\delta}} + \sqrt{2 \mathrm{Var} [\Psi_2] \log \tfrac{2}{\delta}}  + 2 \| A \|_{2} \log \tfrac{2}{\delta} \nonumber \\[0.1cm]  
& = - \KL{\, \for{p}{T-t}^{[j]}}{\, \for{p}{T-t}^{[i]}}
+ \sqrt{2 \mathscr{T}_D^{[i, j]} (t) \log \tfrac{2}{\delta}}
+ \sqrt{2 \mathscr{U}_D^{[i, j]} (t) \log \tfrac{2}{\delta}}  
+  2 \mathscr{V}_D^{[i, j]} (t) \log \tfrac{2}{\delta},  \nonumber 
\end{align*}
where we used: $\mathrm{Var} [\Psi_2] = 2 \mathrm{tr} (A^2) \equiv \mathscr{U}_D^{[i, j]}(t)$ and 
\begin{align*}
\Psi_3 + \mathrm{tr} (A) =  
- \KL{\, \for{p}{T-t}^{[j]}}{\, \for{p}{T-t}^{[i]}}, 
\end{align*}
which completes the proof of Lemma \ref{lemma:weight_bd_j}. 

\subsection{Proof of Theorem \ref{thm:NA_transition_est}} \label{sec:pf_first_main}

Recall the definition of $C_D^{[i,j]}$ in (\ref{eq:C}). We set $\delta = D^{- q}, \, q \ge 2$ and write $r \equiv \log \tfrac{2K}{\delta} = \Theta \bigl( \log (KD) \bigr)$, and $S = \snr{T-t}$ for convenience. We also write $ \mathscr{T} (t) = \mathscr{T}_D^{[i, j]} (t), \, \mathscr{U} (t) = \mathscr{U}_D^{[i, j]} (t)$ and $\mathscr{V} (t) = \mathscr{V}_D^{[i, j]} (t)$ whose definitions are found in (\ref{eq:C}).
Then, the inequality 
\begin{align} \label{eq:target_condition}
C_D^{[i , j]} (t, \delta/K) \ge \eta \KL{\for{p}{T-t}^{[j]}}{\for{p}{T-t}^{[i]}}
\end{align}
is equivalent to $(1 - \eta) \ge G(t)$ with $G(t) = N(t)/ D(t)$, where 
\begin{align}
N (t) & \equiv  \log \tfrac{\pi_{i}}{\pi_j} 
+ \sqrt{2 \mathscr{T} (t) r }
+ \sqrt{2 \mathscr{U} (t) r }  
+ 2 \mathscr{V} (t) r; \label{eq:Nume} \\   
D (t) & \equiv  \KL{\for{p}{T-t}^{[j]}}{\for{p}{T-t}^{[i]}}. 
\label{eq:Deno}
\end{align}
We will derive a sufficient condition for \eqref{eq:target_condition}. To this end, we consider a simplified upper bound for $G(t)$, denoted by $\widetilde{G}(t)$, that involves only $ S = \snr{T-t}$ as a function of time $t$, and then solve the inequality $1 - \eta \ge \widetilde{G}(t)$, which is sufficient for the original condition (\ref{eq:target_condition}).  
%
%
First, we consider an upper bound of the term $N(t)$. Noticing that $\KL{\, \for{p}{T-t}^{[j]}}{\, \for{p}{T-t}^{[i]}} = \KLM{\, \for{p}{T-t}^{[j]}}{\, \for{p}{T-t}^{[i]}} + \KLC{\, \for{p}{T-t}^{[j]}}{\, \for{p}{T-t}^{[i]}}$ with 
\begin{align}
\KLM{\, \for{p}{T-t}^{[j]}}{\, \for{p}{T-t}^{[i]}} 
& := \tfrac{m_{T-t}^2}{2} \| \mu_i - \mu_j \|^2_{(S_{T-t}^i)^{-1}};  \label{eq:klm} \\[0.1cm]
\KLC{\for{p}{T-t}^{[j]}}{\for{p}{T-t}^{[i]}}
& := \frac{1}{2} \Bigl( \tr \bigl( S_{T-t}^j  (S_{T-t}^{i})^{-1} \bigr) - D - \log \tfrac{\det S_{T-t}^j}{\det S_{T-t}^i} \Bigr),  \label{eq:klc}
\end{align}
we deduce from the bound (\ref{eq:T_bd}) in Lemma \ref{lemma:bd_vars} that 
\begin{align}
\sqrt{2 \mathscr{T} (t) r } 
& \le \sqrt{4 r  \bigl(1 + S \cdot \| \Sigma_i - \Sigma_j \|_2 \bigr)
\KLM{\for{p}{T-t}^{[j]}}{\for{p}{T-t}^{[i]}}} \nonumber \\ 
& \le \tfrac{2 r \bigl(1 + S \cdot \| \Sigma_i - \Sigma_j \|_2 \bigr)}{(1 - \eta)} 
+ \tfrac{1-\eta}{2} \KLM{\for{p}{T-t}^{[j]}}{\for{p}{T-t}^{[i]}}, \label{eq:bd_1}
\end{align}
where we used Arithmetic/Geometric Mean inequality:  $\sqrt{a b} \le \tfrac{a}{2 (1- \eta)} + \tfrac{(1 - \eta) b}{2}, \, a, b > 0$. Similarly, we obtain from Lemma \ref{lemma:trace_squared_bd} in Appendix \ref{sec:aux} that 
\begin{align}
\sqrt{2 \mathscr{U} (t) r }   
\le \tfrac{2 r \bigl(1 + S \cdot \| \Sigma_j \|_2 \bigr)}{(1 - \eta)} 
+ \tfrac{1-\eta}{2} \KLC{\for{p}{T-t}^{[j]}}{\for{p}{T-t}^{[i]}}.  \label{eq:bd_2}
\end{align}
Substitution of the bounds (\ref{eq:bd_1}--\ref{eq:bd_2}) and (\ref{eq:V_bd}) into (\ref{eq:Nume}) with Assumptions \ref{assump:prior}, \ref{assump:cov}, \ref{assump:mean} leads to 
\begin{align*}
N (s) & \le \Cone  + \tfrac{4 r}{1 - \eta}  +  \Ctwo D^\alpha \Bigl( \tfrac{10 - 4 \eta}{1 -\eta} \Bigr) 
r  \cdot S + \tfrac{1- \eta}{2} \KL{\for{p}{T-t}^{[j]}}{\for{p}{T-t}^{[i]}}, 
\end{align*}
where the constants $\Cone, \Ctwo $ are collected in Assumptions \ref{assump:prior} and \ref{assump:cov}. Thus, we obtain 
\begin{align*}
G(t) 
& \le \tfrac{\Cone  + \tfrac{4 r}{1 - \eta}  +  \Ctwo D^\alpha \Bigl( \tfrac{10 - 4 \eta}{1 -\eta} \Bigr) 
r  \cdot S}{\KL{\for{p}{T-t}^{[j]}}{\for{p}{T-t}^{[i]}}} + \tfrac{1- \eta}{2}  \le \tfrac{\Cone  + \tfrac{4 r}{1 - \eta}  +  \Ctwo D^\alpha \Bigl( \tfrac{10 - 4 \eta}{1 -\eta} \Bigr) 
r  \cdot S}{\tfrac{S}{2} \Bigl(  \tfrac{\Cfour}{\Csix} \Bigr)^2 \cdot \tfrac{D^2}{S \cdot \tr (\Sigma_i) 
+ D}} + \tfrac{1- \eta}{2} \equiv \widetilde{G}(t),   
\end{align*}
where in the second inequality we used Lemma \ref{lemma:klm_lbd} under Assumptions \ref{assump:mean} and \ref{assump:non-concentration}, provided in Appendix \ref{sec:aux}. Then, $1-\eta \ge G(t)$ is satisfied if $1-\eta \ge \widetilde{G}(t)$ holds, which is simplified as:
\begin{align*}
\tfrac{1-\eta}{2} \cdot \tfrac{\widetilde{C}^2 {D}^2}{S \cdot \tr (\Sigma_i) + {D}} \cdot \tfrac{S}{2} \ge \mathscr{B} + \mathscr{C} S, 
\end{align*}
where we have set: 
\begin{equation*}
\mathscr{B} \equiv \Cone + \tfrac{4r}{1-\eta} > 0, \quad \mathscr{C} \equiv \Ctwo D^\alpha \left(\tfrac{10-4\eta}{1-\eta}\right) r > 0, \quad \widetilde{C} \equiv \tfrac{\Cfour}{\Csix} > 0. 
\end{equation*} 
Rearranging the terms with respect to $S$, we obtain the following quadratic inequality:
\begin{equation} \label{eq:quad_S}
\mathscr{C} \, \tr (\Sigma_i) \, S^2 - \mathscr{A} S + \mathscr{B} {D} \le 0, \quad \text{where} \quad \mathscr{A} \equiv \tfrac{1-\eta}{4} \widetilde{C}^2 {D}^2 
- \mathscr{C} {D} - \mathscr{B} \, \tr (\Sigma_i).
\end{equation} 
Under Assumption \ref{assump:cov}, we have
\begin{align} \label{eq:trace_bd}
\tr (\Sigma_i)  = \sum_{k = 1}^D \lambda_k (\Sigma_i) \le 
\Ctwo M_i D^\alpha + \nu_i \le \Ctwo (D^\alpha + R) M^\dagger
\end{align} 
where the constants $\Ctwo > 0, \nu_i \ge 0, \, R \in [0, 1)$ are collected in Assumption \ref{assump:cov}. Thus, there exists $D_1 \in \mathbb{N}$ such that for all $D \ge D_1$, we have the discriminant strictly positive, i.e.,
\begin{align*} 
\Delta \equiv \mathscr{A}^2 - 4  \mathscr{C} \tr (\Sigma_i) \mathscr{B} {D} > 0. 
\end{align*}
This is because $\mathscr{A}^2 = \Theta(D^4)$ and $\mathscr{C} \tr (\Sigma_i) \mathscr{B}  {D} = \mathcal{O} (D^{1 + 2\alpha} (\log D)^2)$, and thus the leading term dominates the lower-order term for any $\alpha \in [0, 1)$. 
Under $\Delta > 0$, solving \eqref{eq:quad_S} under $S \ge \snr{T}$ leads to $S \in \bigl[ \max \{ \snr{T}, \underline{S} \}, \overline{S} \bigr]$, where 
\begin{align*}
\underline{S} = \tfrac{\mathscr{A} - \sqrt{\Delta}}{2 \mathscr{C} \tr (\Sigma_i)}, \qquad  
\overline{S} = \tfrac{\mathscr{A} + \sqrt{\Delta}}{2 \mathscr{C} \tr (\Sigma_i)}.  
\end{align*} 
First, we note from (\ref{eq:trace_bd}) that for all sufficiently large $D$,
\begin{align*}
\overline{S} > \tfrac{\mathscr{A}}{2 \mathscr{C} \tr (\Sigma_i)} 
\gtrsim \tfrac{D^2}{D^{2 \alpha} \log D} = \tfrac{D^{2 - 2\alpha}}{\log D}, 
\end{align*} 
where the right-hand side is an increasing function of $D$ since $\alpha < 1$. Therefore, we have: for any fixed early stopping criterion $\varepsilon \in (0,1)$, there exists $D_2$ such that for all $D \ge D_2$, the upper root $\overline{S}$ is strictly larger than the early stopping threshold $S_\varepsilon = \snr{\varepsilon} = m_\varepsilon^2 / \sigma_\varepsilon^2$, which is independent of $D$. 
%
Notice that $S = \snr{T-t}$ is a non-decreasing function of the physical time $t$. Thus, we have: for all $D \ge D^\dagger = \max\{D_1, D_2\}$, the target condition (\ref{eq:target_condition}) holds for all $t \in [0, T-\varepsilon]$ satisfying $\snr{T-t} \in \bigl[ \max \{\snr{T}, \underline{S} \}, \snr{\varepsilon} \bigr]$, equivalently to $t \in \bigl[ \max \{0, T + \tfrac{1}{2} \log \tfrac{\underline{S}}{1 + \underline{S}} \}, T- \varepsilon \bigr]$. 

We now study the scale of the threshold $\underline{S}$. We multiply the numerator and denominator by the conjugate $\mathscr{A} + \sqrt{\Delta}$ and have from $\Delta > 0$ and the definition of terms $\mathscr{B}$ and $\mathscr{A}$ that 
\begin{align} \label{eq:s_minus_conjugate}
\underline{S} = \tfrac{2 \mathscr{B} {D}}{\mathscr{A} + \sqrt{\Delta}} \le \tfrac{2 \mathscr{B} {D}}{\mathscr{A}} = \mathcal{O} \left( \tfrac{D \log (KD)}{D^2} \right) = \mathcal{O} \left( \tfrac{\log (KD)}{D} \right), 
\end{align}
where we used $r = \Theta \bigl( \log (KD) \bigr)$. Similarly, we have $\underline{S} \ge \tfrac{ \mathscr{B} {D}}{\mathscr{A}} \gtrsim \tfrac{\log (KD)}{D}$. The first part of the statement has now been proved. 

Finally, we prove the second part of the statement. For any $\beta \in (0,1)$, directly solving $\exp \bigl( - \eta \KL{\for{p}{T-t}^{[j]}}{\for{p}{T-t}^{[i]}} \bigr) \le \beta$ gives 
\begin{align} \label{eq:KL_beta}
\KL{\for{p}{T-t}^{[j]}}{\for{p}{T-t}^{[i]}} \ge \tfrac{1}{\eta} \log \tfrac{1}{\beta}
\end{align} 
under $t \in \bigl[ \max \bigl\{0, T + \tfrac{1}{2} \log \tfrac{\underline{S}}{1 + \underline{S}} \bigr\}, T-\varepsilon \bigr]$. We see that the inequality (\ref{eq:KL_beta}) is satisfied if the lower bound of left-hand-side given in Lemma \ref{lemma:klm_lbd} is no lower than $\tfrac{1}{\eta} \log \tfrac{1}{\beta}$, that is, $\tfrac{\snr{T-t}}{2} \Bigl(  \tfrac{\Cfour}{\Csix} \Bigr)^2 \cdot \tfrac{D^2}{\snr{T-t} \cdot \tr (\Sigma_i) + D} \ge \tfrac{1}{\eta} \log \tfrac{1}{\beta}$. Solving the latter inequality with respect to $\snr{T-t}$ under the range $\snr{T-t} \in [\underline{S}, \snr{\varepsilon}]$ leads to \eqref{eq:snr_transition}, and then the second part of the statement has been shown due to the high probability bound \eqref{eq:key_hpb}. Proof of Theorem \ref{thm:NA_transition_est} is now complete.   
\section{Auxiliary Results for Theorem \ref{thm:NA_transition_est}} \label{sec:aux}

\subsection{First Result} 
\begin{lemma} \label{lemma:trace_squared_bd}
Let $i,k \in [K]$ and $Y = (S_{T-t}^i)^{1/2} (S_{T-t}^k)^{-1} (S_{T-t}^i)^{1/2} - I_D$. We have:
\begin{equation*}
\mathrm{tr}(Y^2) \le 4 \Bigl( 
\snr{T-t} \| \Sigma_i \|_2 + 1 \Bigr) \cdot \KLC{\for{p}{T-t}^{[i]}}{\for{p}{T-t}^{[k]}}, 
\end{equation*}
where $\KLC{\for{p}{T-t}^{[i]}}{\for{p}{T-t}^{[k]}}$ is defined in (\ref{eq:klc}). 
\end{lemma}

\begin{proof}
Let $\{ \eta_l \}_{l= 1, \ldots, D}$ be the eigenvalues of $(S_{T-t}^i)^{1/2} (S_{T-t}^k)^{-1} (S_{T-t}^i)^{1/2}$. Then, the eigenvalues of $Y$ are $\{ \eta_l - 1\}_l$, yielding $\mathrm{tr}(Y^2) = \sum_{l=1}^D (\eta_l - 1)^2$. 
Because $S_{T-t}^i  (S_{T-t}^k)^{-1}$ is similar to $(S_{T-t}^i)^{1/2} (S_{T-t}^k)^{-1} (S_{T-t}^i)^{1/2}$, they share the identical eigenvalues $\eta_l$. We thus get 
$
\KLC{\for{p}{T-t}^{[i]}}{\for{p}{T-t}^{[k]}}
= \frac{1}{2} \sum_{l=1}^D 
(\eta_l - 1 - \log \eta_l).
$
The maximum eigenvalue $\eta_{1}$ is bounded by $\left\| S_{T-t}^i \right\|_2 \left\| (S_{T-t}^{k})^{-1} \right\|_2 \le C_\eta(t)$ with $C_\eta(t) \equiv (\snr{T-t} \| \Sigma_i \|_2 + 1)$. By elementary calculus, the algebraic inequality $(\eta - 1)^2 \le 2 C_\eta(t) (\eta - 1 - \log \eta)$ holds strictly for all $\eta \in (0, C_\eta(t)]$. Applying this inequality to each eigenvalue $\eta_l$ and summing them completes the proof. 
\end{proof}
\subsection{Second Result} \label{sec:pf_klm_lbd}
\begin{lemma}[Upper and lower bounds for KL divergence] \label{lemma:klm_lbd}
It holds under Assumptions \ref{assump:mean} and \ref{assump:non-concentration} that: for any $i, j \in [K]$, 
%
%
\begin{align} 
\begin{aligned} \label{eq:kl_bds}
& \tfrac{\snr{T-t}}{2} \cdot \Bigl(\tfrac{\Cfour}{\Csix} \Bigr)^2
\tfrac{D^2}{\snr{T-t} \tr (\Sigma_i)  +  D} \\[0.2cm] 
& \qquad \le  \, 
\KL{\for{p}{T-t}^{[j]}}{\for{p}{T-t}^{[i]}} 
\le \, \tfrac{\snr{T-t}}{2} \cdot \Cfive D 
+ \tfrac{\snr{T-t}^2}{4} \cdot \| \Sigma_i - \Sigma_j \|_F^2,  
\end{aligned} 
\end{align} 
where the constants $\Cfour, \Cfive, \Csix > 0$ are collected in Assumptions \ref{assump:mean} and \ref{assump:non-concentration}. 
\end{lemma}
\begin{proof} We begin with proving the lower bound of (\ref{eq:kl_bds}). Since the covariance matrix $\Sigma_i$ is symmetric due to Assumption \ref{assump:cov}, its eigen-decomposition gives $\Sigma_i = U^{(i)} \Lambda^{(i)} U^{(i), \top}$ with  the orthogonal matrix $U^{(i)} \equiv [u_1(i), \ldots, u_D (i)]$ and $\Lambda^{(i)} = \mathrm{diag} \bigl( \lambda_1(\Sigma_i), \ldots, \lambda_D(\Sigma_i) \bigr)$. 
For notational simplicity, we write $v_k = \langle \mu_j - \mu_i, u_k (i) \rangle$. Then, $S_{T-t}^i = U^{(i)} \bigl( m_{T-t}^2 \Lambda^{(i)} + \sigma_{T-t}^2 I_D \bigr) U^{(i), \top}$ and (\ref{eq:klm}) yield that
\begin{align} 
\KL{\for{p}{T-t}^{[j]}}{\for{p}{T-t}^{[i]}} \ge 
\KLM{\for{p}{T-t}^{[j]}}{\for{p}{T-t}^{[i]}} 
= \tfrac{m_{T-t}^2}{2} \sum_{k = 1}^D \tfrac{v_k^2}{m_{T-t}^2 \lambda_{k} (\Sigma_i) + \sigma_{T-t}^2}. 
\label{eq:lbd_1}
\end{align}
Further, the Cauchy-Schwarz inequality gives
\begin{align}
\sum_{k = 1}^D \tfrac{v_k^2}{ m_{T-t}^2 \lambda_{k} (\Sigma_i) + \sigma_{T-t}^2} 
\ge \tfrac{\Bigl( \sum_{k = 1}^D | v_k | \Bigr)^2}{\sum_{k = 1}^D (m_{T-t}^2 \lambda_k (\Sigma_i) + \sigma_{T-t}^2)} = 
\tfrac{\Bigl( \sum_{k = 1}^D | v_k | \Bigr)^2}{m_{T-t}^2 \mathrm{tr} (\Sigma_i) + \sigma_{T-t}^2 D}. 
\label{eq:lbd_2} 
\end{align}
For the numerator in the above lower bound, we have $\sum_{k = 1}^D | \nu_k | \ge \tfrac{\Cfour}{\Csix} D$ under Assumption \ref{assump:mean}, which comes from  
\begin{align*} 
\sum_{k = 1}^D v_k^2 = \| \mu_i - \mu_j \|^2 \ge \Cfour D, \quad \sum_{k = 1}^D v_k^2 \le \Csix \sum_{k = 1}^D | v_k |.  
\end{align*}
We thus obtain the desired lower bound by recalling the definition of SNR: $\snr{T-t} = m_{T-t}^2 / \sigma_{T-t}^2$. 
We now show the upper bound of (\ref{eq:kl_bds}). The upper bound for the mean part of the KL (\ref{eq:klm}) is immediately obtained from (\ref{eq:lbd_1}) with Assumption \ref{assump:mean} as:
\begin{align*}
\KLM{\for{p}{T-t}^{[j]}}{\for{p}{T-t}^{[i]}} 
\le \tfrac{m_{T-t}^2}{2 \sigma_{T-t}^2} 
\sum_{k = 1}^D v_k^2 = \tfrac{\snr{T-t}}{2} \| \mu_i - \mu_j \|^2 \le \tfrac{\snr{T-t}}{2} \cdot \Cfive D. 
\end{align*}
We now consider the covariance part of the KL divergence, given in (\ref{eq:klc}). Introducing  $\widetilde{S}_\lambda = I_D + \lambda (S_{T-t}^i 
(S_{T-t}^{k})^{-1} - I_D), \, \lambda \in [0, 1]$, we use standard matrix calculus \citep{Petersen:08} to get: 
\begin{align*}
\log \det \bigl( S_{T-t}^i 
(S_{T-t}^{k})^{-1} \bigr) 
= \int_0^1 \tfrac{d}{d \lambda} \log \det \widetilde{S}_\lambda \, d\lambda = \int_0^1 
\tr \Bigl[ (\widetilde{S}_\lambda)^{-1} \bigl( S_{T-t}^i 
(S_{T-t}^{k})^{-1} - I_D \bigr) \Bigr] d \lambda.  
\end{align*}
Thus, (\ref{eq:klc}) is expressed as:
\begin{align}
\KLC{\for{p}{T-t}^{[i]}}{\for{p}{T-t}^{[k]}} 
& = \tfrac{1}{2} 
\int_0^1 
\tr \Bigl[ \bigl(I_D - (\widetilde{S}_\lambda)^{-1} \bigr) \bigl( S_{T-t}^i 
(S_{T-t}^{k})^{-1} - I_D \bigr) \Bigr] d \lambda 
\nonumber \\
& = \tfrac{1}{2} \int_0^1 \lambda \cdot \tr \Bigl[ (\widetilde{S}_\lambda)^{-1} \bigl( S_{T-t}^i 
(S_{T-t}^{k})^{-1} - I_D \bigr)^2 \Bigr] d \lambda. \label{eq:kl_cov_int}
\end{align}
The expression $\widetilde{S}_\lambda = \bigl(\lambda S_{T-t}^i + (1-\lambda) S_{T-t}^k  \bigr) (S_{T-t}^k)^{-1}$ and the cycliciy of the trace operator yield 
\begin{align}
& \tr \Bigl[ (\widetilde{S}_\lambda)^{-1} \bigl( S_{T-t}^i 
(S_{T-t}^{k})^{-1} - I_D \bigr)^2 \Bigr] \nonumber \\[0.1cm] 
& = \tr \Bigl[ \bigl(\lambda S_{T-t}^i + (1-\lambda) S_{T-t}^k  \bigr)^{-1}  (S_{T-t}^i - S_{T-t}^{k}) (S_{T-t}^k)^{-1} (S_{T-t}^i - S_{T-t}^{k})\Bigr] \nonumber \\[0.1cm] 
& \le \| S_{T-t}^i - S_{T-t}^{k} \|_F^2 \cdot \| (S_{T-t}^k)^{-1} \|_2 \cdot \| (\lambda S_{T-t}^i + (1-\lambda) S_{T-t}^k  )^{-1} \|_2 
\le  \| S_{T-t}^i - S_{T-t}^{k} \|_F^2 \cdot \tfrac{1}{\sigma_{T-t}^4}, \label{eq:cov_tr_bd}
\end{align}  
where we used $(S_{T-t}^k)^{-1} \preceq \tfrac{1}{\sigma_{T-t}^2} I_D$ and the operator convexity of $z \mapsto z^{-1}$ for positive definite matrices to get 
\begin{align*}
(\lambda S_{T-t}^i + (1-\lambda) S_{T-t}^k  )^{-1} \preceq \lambda (S_{T-t}^i)^{-1} + (1-\lambda) (S_{T-t}^k)^{-1} \preceq \tfrac{1}{\sigma_{T-t}^2} I_D. 
\end{align*}
Substitution of (\ref{eq:cov_tr_bd}) into (\ref{eq:kl_cov_int}) together with $S_{T-t}^i - S_{T-t}^k = m_{T-t}^2 (\Sigma_i - \Sigma_k)$ yields
\begin{align*}
\KLC{\for{p}{T-t}^{[i]}}{\for{p}{T-t}^{[k]}} \le 
\tfrac{\snr{T-t}^2}{4} \| \Sigma_i - \Sigma_k \|_F^2,  
\end{align*} 
which completes the proof.
\end{proof}
\section{Proof of Lemma \ref{lemma:bd_vars}} \label{sec:pf_bd_vars}
\textit{Proof of (\ref{eq:V_bd})}. We start with the easiest term $\mathscr{V}_D^{[i,j]} (t) = \| I_D - S_{T-t}^i (S_{T-t}^j)^{-1} \|_2$. We have 
\begin{align} \label{eq:rel_err}
I_D - S_{T-t}^i (S_{T-t}^j)^{-1} = \bigl(S_{T-t}^j - S_{T-t}^i  \bigr) (S_{T-t}^j)^{-1} = m_{T-t}^2 \bigl( \Sigma_j - \Sigma_i \bigr) (S_{T-t}^j)^{-1}
\end{align}
and immediately get the desired bound from $(S_{T-t}^j)^{-1} \preceq \tfrac{1}{\sigma_{T-t}^2} I_D$. \\ 

\noindent 
\textit{Proof of (\ref{eq:T_bd})}. We have 
\begin{align}
\mathscr{T}_D^{[i, j]} (t) & := m_{T-t}^2 \bigl\| \mu_{i} - \mu_{j} \bigr\|_{(S_{T-t}^{i})^{-1} (S_{T-t}^{j}) (S_{T-t}^{i})^{-1}}^2 \nonumber \\ 
& =  m_{T-t}^2 \bigl\| \mu_{i} - \mu_{j} \bigr\|_{(S_{T-t}^i)^{-1}}^2 + 
m_{T-t}^2 \bigl\| \mu_{i} - \mu_{j} \bigr\|_{(S_{T-t}^{i})^{-1} (S_{T-t}^{j}) (S_{T-t}^{i})^{-1} - (S_{T-t}^i)^{-1}}^2.  \label{eq:T_decomp}  
\end{align}
and 
\begin{align*} 
& (S_{T-t}^{i})^{-1} (S_{T-t}^{j}) (S_{T-t}^{i})^{-1} - (S_{T-t}^i)^{-1} \\[0.1cm]
& \qquad = m_{T-t}^2 (S_{T-t}^{i})^{-1} \bigl( \Sigma_j - \Sigma_i \bigr) (S_{T-t}^{i})^{-1}
\preceq  \snr{T-t} \| \Sigma_i - \Sigma_j \|_2 (S_{T-t}^i)^{-1}.  
\end{align*} 
Substitution of this into (\ref{eq:T_decomp}) yields the desired bound. \\

\noindent 
\textit{Proof of (\ref{eq:U_bd}}). The formula (\ref{eq:rel_err}) gives 
\begin{align*}
\mathscr{U}_D^{[i, j]} (t) 
& := \tfrac{1}{2} \mathrm{tr} \bigl( \bigl( I_D - S_{T-t}^{j} (S_{T-t}^{i})^{-1} \bigr)^2 \bigr) = \tfrac{m_{T-t}^4}{2} \mathrm{tr}  \Bigl( 
\bigl( \Sigma_i - \Sigma_j \bigr)^2 (S_{T-t}^{i})^{-2} \Bigr) 
\le \tfrac{\snr{T-t}^2}{2} \| \Sigma_i - \Sigma_j \|_F^2,  
\end{align*}
which completes the proof of Lemma \ref{lemma:bd_vars}. 

\section{Proofs of the Auxiliary Results for Theorem \ref{thm:main}}
\subsection{Proof of Lemma \ref{lemma:PC}} \label{sec:pf_PC}
From the definition of posterior covariance, we have  
$ \Sigma_{T-t} (y) \equiv 
\mathrm{Cov} [\for{X}{0} | \for{X}{T-t} = y] 
= \mathbb{E} [(\for{X}{0})^{\otimes 2} \, | \, \for{X}{T-t} = y] - \bigl( \mathbb{E} [\for{X}{0} \, | \, \for{X}{T-t} = y]  \bigr)^{\otimes 2}
$. Using the posterior weights and \eqref{eq:posterior_mean_cluster}, we express the conditional means as: 
\begin{align*}
\mathbb{E} [(\for{X}{0})^{\otimes 2} \, | \, \for{X}{T-t} = y]
& =  \sum_{k = 1}^K W_D^{[k]} (t, y) \cdot \Bigl\{ \mathrm{Cov} [\for{X}{0} \, | \, \for{X}{T-t} = y, \, \for{X}{0} \sim P_{\mathrm{data}}^{[k]}] + \bigl( \mu^{[k]} (t, y)  \bigr)^{\otimes 2} \Bigr\}; \\[0.2cm] 
\mathbb{E} [\for{X}{0} \, | \, \for{X}{T-t} = y]  & = \sum_{k = 1}^K W_D^{[k]} (t, y) \mu^{[k]} (t, y). 
\end{align*} 
Substituting these into the formula of $\Sigma_{T-t} (y)$ gives  
\begin{align*}
\Sigma_{T-t} (y) & = \sum_{k = 1}^K W_D^{[k]} (t, y) \, \mathrm{Cov} [\for{X}{0} \, | \, \for{X}{T-t} = y, \, \for{X}{0} \sim P_{\mathrm{data}}^{[k]}] 
+ \sum_{k = 1}^K W_D^{[k]} (t, y) \bigl( \mu^{[k]} (t, y) \bigr)^{\otimes 2}  \\ 
& \qquad 
- \sum_{k_1, k_2 = 1}^K W_D^{[k_1]} (t, y) W_D^{[k_2]} (t, y) \mu^{[k_1]}(t, y) \mu^{[k_2]}(t, y)^\top  \\ 
& = \sum_{k = 1}^K W_D^{[k]} (t, y) \, \mathrm{Cov} [\for{X}{0} \, | \, \for{X}{T-t} = y, \, \for{X}{0} \sim P_{\mathrm{data}}^{[k]}] \\ 
& \qquad 
+ \frac{1}{2} \sum_{k_1, k_2 = 1}^K W_D^{[k_1]} (t, y) W_D^{[k_2]} (t, y) \bigl( \mu^{[k_1]}(t, y) - \mu^{[k_2]}(t, y) \bigr)^{\otimes 2}. 
\end{align*}
Finally, the conditional covariance formula for joint Gaussian distributions deduces that 
\begin{align*} 
\mathrm{Cov} [\for{X}{0} \, | \, \for{X}{T-t} 
= y, \, \for{X}{0} \sim P_{\mathrm{data}}^{[k]}] 
= \Sigma_k - m_{T-t}^2 \Sigma_k (S_{T-t}^k)^{-1} \Sigma_k 
= \sigma_{T-t}^2 \Sigma_k (S_{T-t}^k)^{-1}, 
\end{align*}
which concludes Lemma \ref{lemma:PC}. 

\subsection{Proof of Lemma \ref{lemma:max_chi_squared}} \label{sec:pf_max_chi_squared}
First, we have for $\rho > 0$ 
\begin{align*}
& \log \Bigl(  \mathbb{E} \Bigl[ \max_{k_1, k_2 \in [K]}  \exp \Bigl( \rho \cdot \bigl\| \mu^{[k_1]} (Z_j^{[i]})
- \mu^{[k_2]} (Z_j^{[i]}) \bigr\|^2 \Bigr) \Bigr]  \Bigr) 
\nonumber \\ 
& \le \log \Bigl( \sum_{k_1, k_2 \in [K]} \mathbb{E} \bigl[ \exp \bigl( \rho \cdot \bigl\| \mu^{[k_1]} (Z_j^{[i]})
- \mu^{[k_2]} (Z_j^{[i]}) \bigr\|^2 \bigr) \bigr]  \Bigr) \nonumber \\ 
& \le \log \Bigl( \sum_{k_1, k_2 \in [K]} \prod_{k \in \{k_1, k_2\} }
\sqrt{ \mathbb{E} \bigl[ \exp \bigl( 4 \rho \cdot \bigl\| \mu^{[k]} (Z_j^{[i]})
- \mu^{[i]} (Z_j^{[i]}) \bigr\|^2 \bigr) \bigr]  } \Bigr). 
\end{align*} 
where in the last line we used $ \| \mu^{[k_1]} (Z_j^{[i]}) - \mu^{[k_2]} (Z_j^{[i]}) \|^2 \le 2 \bigl\{\| \mu^{[k_1]} (Z_j^{[i]}) - \mu^{[i]} (Z_j^{[i]}) \|^2 
+ \| \mu^{[k_2]} (Z_j^{[i]}) - \mu^{[i]} (Z_j^{[i]}) \|^2  \bigr\}$ and Cauchy-Schwartz inequality.  
For notational simplicity, we write $r_j = {\sigma_{T-t_j}^2}/{m_{T-t_j}}$ and recall $\mu^{[k]} (t, x) = \mathbb{E} [\for{X}{0} | \for{X}{T-t} = x, \ \for{X}{0} \sim P_{\mathrm{data}}^{[k]}]$. Then, Tweedie's formula for $\nabla \log \for{p}{T-t}^{[k]} (x), \, (t, x) \in [0, T) \times \mathbb{R}^D$ gives 
$$
\mu^{[k]} (t, x) = \tfrac{1}{m_{T-t}} \cdot \bigl( x + \sigma_{T-t}^2 \nabla \log \for{p}{T-t}^{[k]} (x) \bigr)
$$ and then
\begin{align}
\mu^{[k]} (Z_j^{[i]}) - \mu^{[i]} (Z_j^{[i]})
& = r_j \Bigl( \nabla \log \for{p}{T-t_j}^{[k]} (\for{X}{T-t_j}^{[i]}) 
- \nabla \log \for{p}{T-t_j}^{[i]} (\for{X}{T-t_j}^{[i]}) \Bigr)  \nonumber \\ 
& \overset{d}{=} r_j  \cdot \bigl( F_{k, i} (j) + G_{k, i} (j) \bigr) \label{eq:posterior_mean_diff} 
\end{align} 
%
%
%
where 
\begin{align*}
F_{k, i} (j) = m_{T-t_j} \cdot (S_{T-t_j}^{k})^{-1} 
(\mu_{k} - \mu_i) , \quad 
G_{k, i} (j) \sim \mathcal{N} \bigl( 0,  H_{k, i} (j) \bigr) 
\end{align*}
with $ H_{k, i} (j)$ defined in \eqref{eq:H}. The moment generating function for the quadratic of Gaussians under the expression \eqref{eq:posterior_mean_diff} gives  
\begin{align}
& \log \mathbb{E} \Bigl[ \exp \Bigl( 4 \rho \cdot \bigl\| \mu^{[k]} (Z_j^{[i]}) - \mu^{[i]} (Z_j^{[i]}) \bigr\|^2 \Bigr) \Bigr]  \nonumber \\ 
& \qquad = - \tfrac{1}{2} \log \det \bigl( I_D - 8 \rho r_j^2 H_{k, i} (j) \bigr) 
+ 4 \rho r_j^2 (F_{k, i} (j))^\top \bigl( I_D - 8 \rho r_j^2 H_{k, i} (j) \bigr)^{-1} F_{k, i} (j)  \label{eq:quad_gauss}
\end{align}
for any $\rho > 0 $ such that $I_D - 8 \rho r_j^2 H_{k, i} (j)$ is positive definite, which is satisfied if $\rho < 1/ (8 r_j^2 \| H_{k, i} (j) \|_2)$. Thus, under $\rho \le 1/ (16 r_j^2 \max_{k, i \in [K]}\| H_{k, i} (j) \|_2)$, we have 
\begin{align*}
    - \tfrac{1}{2} \log \det \bigl( I_D - 8 \rho r_j^2 H_{k, i} (j) \bigr) 
    = - \tfrac{1}{2} \sum_{\ell = 1}^D \log \bigl( 1 - 8 \rho r_j^2 \, \lambda_\ell (H_{k, i} (j)) \bigr) \le 8 \rho r_j^2 \cdot \mathrm{tr} \bigl( H_{k, i} (j) \bigr), 
\end{align*}
because $- \log (1 - x) \le 2 x$ for $x \in [0, 1/2]$. We have
\begin{align*}
    \tr (H_{k, i} (j)) & = m_{T-t}^4  \tr \bigl( (S_{T-t_j}^k)^{-1} (\Sigma_k - \Sigma_i) (S_{T-t_j}^i)^{-1} (\Sigma_k - \Sigma_i) (S_{T-t_j}^k)^{-1} \bigr) \nonumber \\ 
    & \le \tfrac{m_{T-t_j}^4}{\sigma_{T-t_j}^6} \| \Sigma_k - \Sigma_i \|_F^2 \lesssim \tfrac{m_{T-t_j}^4}{\sigma_{T-t_j}^6} M^\dagger D^{2 \alpha}
\end{align*}
where we used $(S_{T-t_j}^k)^{-1} \preceq \tfrac{1}{\sigma_{T-t}^2} I_D$ in the first inequality and the following bound in the second: 
\begin{align}
\| \Sigma_k - \Sigma_i \|_F^2  
& \lesssim \| \Sigma_k \|_F^2 + \| \Sigma_i \|_F^2 \le 2 \Ctwo D^\alpha (C_2 D^\alpha + R) M^\dagger  \lesssim D^{2\alpha} M^\dagger.  
\label{eq:Sigma_diff_F}  
\end{align}
where the constants $C_2 > 0, \, R \in [0,1)$ are collected in Assumption \ref{assump:cov}. In the above, we used $\| \Sigma_k \|_F^2  = \mathrm{tr} \bigl( \Sigma_k^\top \Sigma_k \bigr) = \mathrm{tr} \bigl( \Lambda_k^\top \Lambda_k \bigr) \le \lambda_1 (\Sigma_k) \cdot \mathrm{tr} (\Lambda_k) \le \Ctwo D^\alpha (\Ctwo D^\alpha M_k + \nu_k) \le \Ctwo D^\alpha (\Ctwo D^\alpha + R) M^\dagger$ under Assumption \ref{assump:cov} with decomposition $\Sigma_k = U \Lambda_k U^\top$. 
We thus have that for 
$\rho \le 1/ (16 r_j^2 \max_{k, i \in [K]}\| H_{k, i} (j) \|_2)$, 
\begin{align} \label{eq:quad_gauss_first}
- \tfrac{1}{2} \log \det \bigl( I_D - 8 \rho r_j^2 H_{k, i} (j) \bigr)  \le \rho \, \snr{j} M^\dagger D^{2 \alpha}. 
\end{align}
Similarly, we also have that for $\rho \le 1/ (16  r_j^2 \max_{k, i \in [K]} \| H_{k, i} (j) \|_2)$,  
\begin{align}  \label{eq:quad_gauss_second} 
\rho r_j^2 (F_{k, i} (j))^\top \bigl( I_D - 8 \rho r_j^2 H_{k, i} (j) \bigr)^{-1} F_{k, i} (j) \le 2 \rho r_j^2 \|  F_{k, i} (j) \|^2 \lesssim 
\rho D. \quad \bigl( \because \mathrm{Assumption} \, \ref{assump:mean} \bigr)  
\end{align}
Substituting \eqref{eq:quad_gauss_first} and \eqref{eq:quad_gauss_second} into \eqref{eq:quad_gauss}, we conclude Lemma \ref{lemma:max_chi_squared}. 
\subsection{Proof of Lemma \ref{lemma:second_moment_bd}} \label{sec:pf_second_moment_bd}
Writing $\for{X}{T-t_j}^{[i]}$ as $m_{T-t_j} \mu_i + Z$ with $Z \sim \mathcal{N} (0, S_{T-t_j}^i)$, we have   
\begin{align}
D_k^{[i]} (j) 
& = \tfrac{\sigma_{T-t_j}^4}{m_{T-t_j}^2} \times \mathbb{E} \Bigl[ \bigl\| m_{T-t_j} (S_{T-t_j}^{k})^{-1} (\mu_i - \mu_{k}) + \bigl\{ (S_{T-t_j}^{k})^{-1} - (S_{T-t_j}^{i})^{-1} \bigr\} Z \bigr\|^2 \Bigr] \nonumber \\[0.1cm] 
& = \tfrac{\sigma_{T-t_j}^4}{m_{T-t}^2} \times \Bigl\{ m_{T-t_j}^2 \|   \mu_i - \mu_{k} \|^2_{(S_{T-t_j}^{k})^{-2}}  
+ \mathrm{tr} \bigl( \bigl\{ (S_{T-t_j}^{k})^{-1} - (S_{T-t_j}^{i})^{-1} ) \bigr\}^2 S_{T-t_j}^i \bigr) \Bigr\},  \label{eq:D_k}
\end{align}
where in the first equation, we used: for any $(t, x) \in [0, T) \times \mathbb{R}^D$, 
\begin{align*}
\mu^{[k]} (T-t, x)
- \mu^{[i]} (T-t, x) 
& = \tfrac{\sigma_{T-t}^2}{m_{T-t}} \bigl( \nabla \log p_{T-t}^{[i]} (x) - \nabla \log p_{T-t}^{[k]} (x) \bigr), 
\end{align*}
with $\nabla \log p_{T-t}^{[k]} (x) = - (S_{T-t}^{k})^{-1} \bigl( x - m_{T-t} \mu_{k} \bigr)$. The matrix inequality $(S_{T-t_j}^i)^{-1} \preceq \tfrac{1}{\sigma_{T-t}^2} I_D$ together with (\ref{eq:klm}) yields 
\begin{align}
m_{T-t_j}^2 \|   \mu_i - \mu_{k} \|^2_{(S_{T-t_j}^{k})^{-2}} 
& \le \tfrac{m_{T-t_j}^2}{\sigma_{T-t_j}^2} \|   \mu_i - \mu_{k} \|^2_{(S_{T-t_j}^{k})^{-1}} 
= \tfrac{2}{\sigma_{T-t_j}^2} \mathrm{KL}_{\mathrm{m}} \bigl(  \for{p}{T-t}^{[i]} \,|| \, 
\for{p}{T-t}^{[k]}  \bigr). \label{eq:KL_mean}
\end{align} 
By setting $Y := (S_{T-t_j}^i)^{1/2} (S_{T-t_j}^{k})^{-1} (S_{T-t_j}^i)^{1/2} - I_D$, we have
\begin{align}
\tr \bigl( \bigl\{ (S_{T-t_j}^{k})^{-1} - (S_{T-t_j}^{i})^{-1} \bigr\}^2 S_{T-t_j}^i \bigr) 
& = \tr \bigl( (S_{T-t_j}^i)^{-1} Y^2 \bigr)  \le \tfrac{1}{\sigma_{T-t_j}^2} \tr (Y^2)  \nonumber \\ 
& \le \tfrac{4}{\sigma_{T-t_j}^2} \Bigl( 
\tfrac{m_{T-t}^2}{\sigma_{T-t}^2} \| \Sigma_i \|_2 + 1 \Bigr) \cdot \mathrm{KL}_{\mathrm{c}} \bigl(  \for{p}{T-t}^{[i]} \,|| \, 
\for{p}{T-t}^{[k]}  \bigr) \label{eq:KL_cov}
\end{align} 
where in the last line we used Lemma \ref{lemma:trace_squared_bd} provided in Appendix \ref{sec:aux}. Substituting the inequalities (\ref{eq:KL_mean}) and (\ref{eq:KL_cov}) into (\ref{eq:D_k}), we conclude.   
\section{Details on Numerical Experiments} \label{sec:details_NE}
This section provides details of the numerical experiments presented in Section \ref{sec:numerical_experiments}, and additional information to support the arguments of the section.   

\subsection{Implementation details}
\paragraph{Data processing}
For the scRNA-seq experiments on PBMC, we use the PBMC3K cell-by-gene count matrix as the raw data
and apply a standard single-cell preprocessing workflow; see e.g., \cite{satija:15} and \cite{stuart:19}. First, we remove genes expressed in fewer than three cells and apply total-count normalisation with a target sum of $10^4$, followed by log-transformation. Then, for each ambient
dimension $D \in \{500,1000,2000,4000,8000\}$, we select the top $D$ highly variable genes using the Seurat-flavour procedure implemented in the Python package \texttt{Scanpy} \citep{wolf:18}. We retain the CD4 T,
CD14+ Monocyte, and CD19+ B populations with the number of samples 1144, 480 and 342 respectively.  We then apply gene-wise standardisation, and the resulting $D$-dimensional gene-expression vectors are used
as the empirical target distribution for the reverse trajectories.

For the image experiment, we use the training split of the STL-10 dataset, whose
images contain $96 \times 96$ pixels with 3 RGB channels. We retain three classes: Dog, Cat, and
Airplane, and use 500 images from each class. To vary the ambient dimension, we
construct three image resolutions \(s \in \{32,64,96\}\), corresponding to
\(D=3s^2\). For \(s=32\) and \(s=64\), the original images are downsampled using
bilinear interpolation implemented by \texttt{torch.nn.functional.interpolate}
in PyTorch. For \(s=96\), we use the original
STL-10 resolution. Finally, pixel intensities are rescaled from $[0,255]$ to
$[-1,1]$.

\paragraph{Calculation of posterior weights and KL divergence} We approximate the true intractable KL divergence between two noised clusters, using a Monte Carlo approximation based on the empirical target distribution. That is, for two distinct clusters $A, B \in \mathcal{C}_{\mathrm{PBMC}}$ or $\mathcal{C}_{\mathrm{image}}$ and $t \in [0, T)$,  
\begin{align*} 
\KL{\for{p}{T-t}^{[A]}}{\for{p}{T-t}^{[B]}}
& = \int \log \tfrac{\for{p}{T-t}^{[A]} (x)}{\for{p}{T-t}^{[B]} (x)} \cdot \for{p}{T-t}^{[A]} (x) dx  \approx \int \log \tfrac{\for{p}{T-t}^{[A], N_A} (x)}{\for{p}{T-t}^{[B], N_B} (x)} \cdot \for{p}{T-t}^{[A], N_A} (x) dx, 
\end{align*} 
where $N_A$ and $N_B$ denote the numbers of empirical samples in classes $A$ and $B$, respectively, and we define  
\begin{align} \label{eq:empirical_forward_density}
\for{{p}}{T-t}^{[i], N_i} (x) 
\equiv \tfrac{1}{N_i} \sum_{j = 1}^{N_i} \for{p}{T-t|0} (x | Z_j), \qquad 
\{ Z_j \}_{j} \overset{\mathrm{i.i.d.}}{\sim} p_{0}^{[i]}, \quad i \in \{A, B\},
\end{align} 
where $p_{0}^{[i]}$ is the population distribution for the class `$i$'. Then, the estimate of the KL is finally given as:
\begin{align*}
\KL{\for{p}{T-t}^{[A]}}{\for{p}{T-t}^{[B]}} 
\approx \tfrac{1}{N_A M} \sum_{l = 1}^{N_A}
\sum_{k = 1}^M 
\log \tfrac{\for{p}{T-t}^{[A], N_A} (X^{k, l})}{\for{p}{T-t}^{[B], N_B} (X^{k, l})} 
\end{align*} 
where $X^{k, l} = m_{T-t} Z_l + \sigma_{T-t} \xi_k, \, \{\xi_k\}_{k= 1, \ldots, M}  \overset{\mathrm{i.i.d.}}{\sim} N(0, I_D), \, 
\{ Z_j \}_{j = 1, \ldots, N_A} \overset{\mathrm{i.i.d.}}{\sim} p_{0}^{[A]}$ and $M$ is the number of MC samples for the white noise. Since the true posterior weight for general target distributions is intractable, we estimate it using the empirical forward density (\ref{eq:empirical_forward_density}) again as follows: for the set of classes $\mathcal{C}$, 
\begin{align*}
\widehat{W}^{[i], N}_D (t, x) = \frac{\for{{p}}{T-t}^{[i], N_i} (x) }{\sum_{c \in \mathcal{C}} \for{{p}}{T-t}^{[c], N_c} (x)}, \quad i \in \mathcal{C}, \quad (t, x)  \in [0, T) \times \mathbb{R}^D,  
\end{align*} 
where $N_c$ denotes the number of samples in the class $c \in \mathcal{C}$, and we assume equal priors across all the classes.  

\paragraph{Reverse SDE with Empirical Score} As mentioned in Section \ref{sec:numerical_experiments}, we generate reverse trajectories using the empirical score function in the drift. Following the notation above, we write the reverse SDE as:  
\begin{align*}
d \back{X}{t} = (\back{X}{t} + 2 \nabla \log \for{p}{T-t}^{N} (\back{X}{t} )) dt + \sqrt{2} d B_t, \qquad 
\back{X}{0} \sim N(0, I_D), 
\end{align*} 
where $N$ is the total number of empirical data, including all classes. In the above, the empirical score is given as 
\begin{align*}
\nabla \log \for{p}{T-t}^{N} (x) 
= \sum_{c \in \mathcal{C}} \widehat{W}_D^{[c], N} (t, x) \cdot  
\nabla \log \for{p}{T-t}^{[c], N_c} (x),  \qquad x \in \mathbb{R}^D, 
\end{align*} 
where $\nabla \log \for{p}{T-t}^{[c], N_c} (x)$ is the cluster-wise empirical score with the number of samples $N_c$ written as 
\begin{align*}
\nabla \log \for{p}{T-t}^{[c], N_c} (x) 
= - \tfrac{1}{\sigma_{T-t}^2} \bigl( x - m_{T-t} {\mathbb{E}}_{N_c} [\for{X}{0}^{[c]} | \for{X}{T-t}^{[c]} = x] \bigr), 
\end{align*} 
with 
\begin{align*}
{\mathbb{E}}_{N_c} 
[\for{X}{0}^{[c]} | \for{X}{T-t}^{[c]} = x] 
= \tfrac{1}{N_c} \sum_{j = 1}^{N_c} Z_j \cdot \tfrac{\for{p}{T-t |0} (x | Z_j)}{\for{p}{T-t}^{[c], N_c} (x)}, \quad \{Z_j\}_{j = 1, \ldots, N_c} \overset{\mathrm{i.i.d.}}{\sim} p_0^{[c]}. 
\end{align*}  
%
%
%
\subsection{Additional Information}
We present the empirical pairwise geometry for the PBMC dataset. We observe that the mean separation increases with the ambient dimension $D$ approximately linearly, while the maximum mean difference projection $M_\mu$ is relatively stable across varying $D$.  
\begin{table}[h]
\centering
\footnotesize
\caption{
Pairwise empirical geometry of the PBMC clusters under variance-ordered feature
truncation. The labels are the same as in Table \ref{tab:stl10_pairwise_geometry}. 
}
\label{tab:pbmc_pairwise_geometry}
\renewcommand{\arraystretch}{1.12}
\begin{tabular}{@{} llrrr @{}}
\toprule
Pair & \(D\)
& \(\|\widehat\mu_i-\widehat\mu_j\|^2\)
& \(M_\mu\)
& \(\|\widehat\Sigma_i-\widehat\Sigma_j\|_F^2\) \\
\midrule

\multirow{5}{*}{CD4 T--CD14+ Monocyte}
& \(500\)  & \(31.55\)  & \(1.48\) & \(1.34{\times}10^3\) \\
& \(1000\) & \(91.95\)  & \(2.67\) & \(4.16{\times}10^3\) \\
& \(2000\) & \(240.40\) & \(4.37\) & \(1.55{\times}10^4\) \\
& \(4000\) & \(352.70\) & \(4.62\) & \(5.35{\times}10^4\) \\
& \(8000\) & \(509.33\) & \(5.68\) & \(1.90{\times}10^5\) \\
\addlinespace

\multirow{5}{*}{CD14+ Monocyte--CD19+ B}
& \(500\)  & \(42.77\)  & \(1.24\) & \(2.40{\times}10^3\) \\
& \(1000\) & \(111.03\) & \(2.39\) & \(7.31{\times}10^3\) \\
& \(2000\) & \(296.89\) & \(3.41\) & \(2.94{\times}10^4\) \\
& \(4000\) & \(440.36\) & \(4.01\) & \(1.16{\times}10^5\) \\
& \(8000\) & \(614.00\) & \(5.29\) & \(4.04{\times}10^5\) \\
\addlinespace

\multirow{5}{*}{CD4 T--CD19+ B}
& \(500\)  & \(20.84\)  & \(1.96\) & \(1.70{\times}10^3\) \\
& \(1000\) & \(53.11\)  & \(1.49\) & \(5.21{\times}10^3\) \\
& \(2000\) & \(133.63\) & \(1.73\) & \(2.09{\times}10^4\) \\
& \(4000\) & \(205.03\) & \(1.60\) & \(8.85{\times}10^4\) \\
& \(8000\) & \(271.85\) & \(1.99\) & \(3.19{\times}10^5\) \\
\bottomrule
\end{tabular}
\end{table}

\end{document}